\documentclass{article}

\PassOptionsToPackage{numbers, compress}{natbib}
\usepackage{amsmath}
\usepackage{microtype}
\usepackage{graphicx}
\usepackage{subfigure}
\usepackage{hyperref}
\hypersetup{
    colorlinks=true,
    linkcolor=blue,
    filecolor=magenta,      
    urlcolor=cyan,
    citecolor=black
}
\usepackage{booktabs} 
\usepackage{amsfonts}
\usepackage{bm}
\usepackage{threeparttable}
\usepackage{tablefootnote}

\usepackage{lipsum}
\usepackage{comment}

\usepackage{multirow}
\usepackage[usenames, dvipsnames]{color}
\usepackage{amsmath}

\usepackage{hyperref}

\usepackage{enumitem}

\usepackage{mathtools}

\usepackage{listings}

\usepackage[english]{babel}
\usepackage[utf8x]{inputenc}

\usepackage{listings}
\usepackage{color}

\definecolor{mygreen}{rgb}{0,0.6,0}
\definecolor{mygray}{rgb}{0.5,0.5,0.5}
\definecolor{mymauve}{rgb}{0.58,0,0.82}

\usepackage{marginnote}
\usepackage{soul} 

\usepackage[colorinlistoftodos]{todonotes}

\newcommand{\ignore}[1]{}



     \usepackage[preprint]{neurips_2021}

\usepackage{amsmath}
\usepackage{amssymb}
\usepackage{amsthm}

\usepackage{mathtools}

\usepackage{tabularx}
\usepackage{multirow}
\usepackage{hhline}
\usepackage{algcompatible}
\usepackage{algorithm}
\usepackage{enumitem}
\usepackage{multicol}
\usepackage{bm}
\usepackage{svg}
\usepackage{titlesec}
\titlespacing\section{0pt}{0pt plus 0pt minus 1pt}{0pt plus 0pt minus 1pt}
\titlespacing\subsection{0pt}{0pt plus 0pt minus 1pt}{0pt plus 0pt minus 1pt}
\titlespacing\subsubsection{0pt}{0pt plus 0pt minus 1pt}{0pt plus 0pt minus 1pt}

\newtheorem{definition}{Definition}
\newtheorem{lemma}{Lemma}
\newtheorem{proposition}{Proposition}

\def\va{{\bm{a}}}
\def\vb{{\bm{b}}}

\def\vk{{\bm{k}}}

\def\vq{{\bm{q}}}

\def\vv{{\bm{v}}}

\def\vx{{\bm{x}}}

\def\mA{{\bm{A}}}

\def\mD{{\bm{D}}}

\def\mK{{\bm{K}}}
\def\mL{{\bm{L}}}

\def\mQ{{\bm{Q}}}

\def\mU{{\bm{U}}}
\def\mV{{\bm{V}}}
\def\mW{{\bm{W}}}
\def\mX{{\bm{X}}}

\def\mZ{{\bm{Z}}}

\def \RR {{\mathbb{R}}}

\usepackage{verbatim}

\usepackage{wrapfig}

\title{
FMMformer: Efficient and Flexible Transformer \\ via Decomposed Near-field and Far-field Attention
}

\author{%
   Tan M. Nguyen\\
  Department of Mathematics\\ 
  University of California, Los Angeles\\
  Los Angeles, CA, USA\\
  \And
   Vai Suliafu \thanks{Co-first author} \\
   School of Computing\\
   Scientific Computing and Imaging (SCI) Institute\\
   University of Utah, Salt Lake City, UT, USA\\
   \And
   Stanley J. Osher \\
   Department of Mathematics\\
   University of California, Los Angeles\\
   Los Angeles, CA, USA\\
    \AND
   Long Chen \\
  Department of Mathematics\\
  University of California, Irvine\\
  Irvine, CA, USA\\
   \And
   Bao Wang \thanks{Please correspond to: wangbaonj@gmail.com  or chenlong@math.uci.edu} \\
   Department of Mathematics\\
   Scientific Computing and Imaging (SCI) Institute\\
   University of Utah, Salt Lake City, UT, USA\\
}

\begin{document}

\setlength{\abovedisplayskip}{2.5pt}
\setlength{\belowdisplayskip}{2.5pt}

\maketitle

\begin{abstract}
We propose FMMformers, a class of efficient and flexible transformers inspired by the celebrated fast multipole method (FMM) for accelerating interacting particle simulation. FMM decomposes particle-particle interaction into near-field and far-field components and then performs direct and coarse-grained computation, respectively. Similarly, FMMformers decompose the attention into near-field and far-field attention, modeling the near-field attention by a banded matrix and the far-field attention by a low-rank matrix. Computing the attention matrix for FMMformers requires linear complexity in computational time and memory footprint with respect to the sequence length. In contrast, standard transformers suffer from quadratic complexity. We analyze and validate the advantage of FMMformers over the standard transformer on the Long Range Arena and language modeling benchmarks. FMMformers can even outperform the standard transformer in terms of accuracy by a significant margin. For instance, FMMformers achieve an average classification accuracy of $60.74\%$ over the five Long Range Arena tasks, which is significantly better than the standard transformer's average accuracy of $58.70\%$.
\end{abstract}

\section{Introduction}

Transformers \cite{vaswani2017attention} have achieved state-of-the-art performance in sequence processing tasks, including machine translation and language modeling \cite{vaswani2017attention,al2019character,dai2019transformer,baevski2018adaptive,williams-etal-2018-broad,devlin2018bert,NEURIPS2020_1457c0d6}. Also, transformers can effectively transfer knowledge from a pre-trained model to tasks with limited supervision \cite{radford2018improving,radford2019language,devlin2018bert,yang2019xlnet,liu2019roberta}. Transformers rely on the attention mechanism and particularly self-attention \cite{cho-etal-2014-learning,parikh-etal-2016-decomposable,DBLP:journals/corr/LinFSYXZB17}, an inductive bias that connects each token in the input through a relevance weighted basis of every other token, as a fundamental building block for their modeling \cite{bahdanau2014neural,vaswani2017attention,kim2017structured}. Moreover, it has been argued that the flexibility in capturing diverse syntactic and semantic relationships \cite{tenney-etal-2019-bert,vig-belinkov-2019-analyzing,clark-etal-2019-bert,voita-etal-2019-analyzing,hewitt-liang-2019-designing} and the capacity of the attention mechanism \cite{tenney-etal-2019-bert} are critical components for the success of transformers.

\subsection{Self-attention}
The self-attention mechanism is used to learn long-range dependencies while enabling parallel processing of the input sequence. For a given input sequence $\mX:=[\vx_1,\vx_2,\cdots,\vx_N]^\top\in \RR^{N\times D_x}$ of $N$ feature vectors that have been encoded in a $D_x$-dimensional vector space, self-attention transforms $\mX$ into an output sequence $\Hat{\mV}$ in the following two steps:

\begin{enumerate}
    \item [Step 1.] Project the input sequence $\mX$ into three matrices via the following linear transformations
    $$
\mQ=\mX\mW_Q^\top; \mK=\mX\mW_K^\top; \mV=\mX\mW_V^\top,
$$
where $\mW_Q,\mW_K\in \RR^{D\times D_x}$, and $\mW_V\in \RR^{D_v\times D_x}$ are the weight matrices. We denote $\mQ:=[\vq_1,\cdots,\vq_N]^\top, \mK:=[\vk_1,\cdots,\vk_N]^\top$, and $\mV:=[\vv_1,\cdots,\vv_N]^\top$, where the vectors $\vq_i,\vk_i,\vv_i$ for $i=1,\cdots,N$ are the query, key, and value vectors, respectively.
    
    \item [Step 2.] For each query vector $\vq_i$ for $i=1,\cdots,N$, we compute the output vector $\Hat{\vv}_i$ as follows
    \begin{equation}\label{eq:attention-vec}
\hat{\vv}_i=\sum_{j=1}^N{\rm softmax}\Big(\frac{{\vq}_i^\top{\vk}_j}{\sqrt{D}}\Big){\vv}_j,\ \Longleftrightarrow 
\hat{\mV}={\rm softmax}\Big(\frac{{\mQ}{\mK}^\top }{\sqrt{D}}\Big){\bf V} :={\mA}{\mV},
\end{equation}
where the softmax function is applied to each row of the matrix $(\mQ\mK^\top)/\sqrt{D}$.
\end{enumerate}

For long sequences, the computational time and memory footprint of transformers are dominated by \eqref{eq:attention-vec}. It is evident that the memory cost is $\mathcal{O}(N^2)$ to store the attention matrix $\mA$. Also, the computational complexities of computing the matrix-matrix products $\mQ\mK^\top$ and $\mA\mV$ are both $\mathcal{O}(N^2)$. These limitations impede the application of transformers to many important settings that involve very long sequences \cite{j.2018generating,huang2018music,pmlr-v80-parmar18a}. When applying self-attention for long sequence modeling, we have to limit the context window to a reasonable size to make it computationally feasible, limiting the effectiveness of learning long-term dependencies. Efficient transformer models have been proposed, including leveraging sparse and low-rank attention. Many of the existing efficient transformers gain computational and memory efficiency at the cost of significant accuracy degradation. 

\subsection{Contribution}
Leveraging the idea of the fast multipole method (FMM) \cite{greengard1987fast}, we propose a class of efficient and flexible transformer, namely \emph{FMMformers}, to boost the performance of efficient transformers. At the core of FMMformers is to replace the self-attention $\Hat{\mV}=\mA\mV$ in \eqref{eq:attention-vec} with the following matrix-matrix product
\begin{equation}\label{eq:FMMformer-attention}
\Hat{\mV} := (\mD+\mL)\mV, 
\end{equation}
where $\mD$ is a banded matrix with bandwidth $k\ll N$ and $\mL$ is a low-rank matrix of rank $r\ll N$. In practice, we normalize matrix $\mD+\mL$ such that the sum of each row is $1$; for the sake of presentation, we ignore this normalization step below. Both $\mD\mV$ and $\mL\mV$ can be computed with linear computational and memory complexity; they model the near-field and far-field attention, respectively. FMMformers are flexible in designing the sparse banded matrix and the low-rank matrix for modeling near-field and far-field attention. In particular, we can control the bandwidth of the banded matrix $\mD$ and the rank of the low-rank matrix $\mL$ for expressivity and efficiency tradeoff. In addition to the efficiency and flexibility, FMMformers gain significant accuracy improvement over linear transformers and can even outperform the standard transformer in terms of accuracy. We illustrate the idea of FMMformers in Figure~\ref{fig:near_far_field}: Instead of modeling the full attention by a dense unstructured matrix, we employ a sparse banded matrix to model the near-field attention and several rank one matrices to model the far-field attention.

\begin{figure*}[!ht]
\centering
\includegraphics[width=\linewidth]{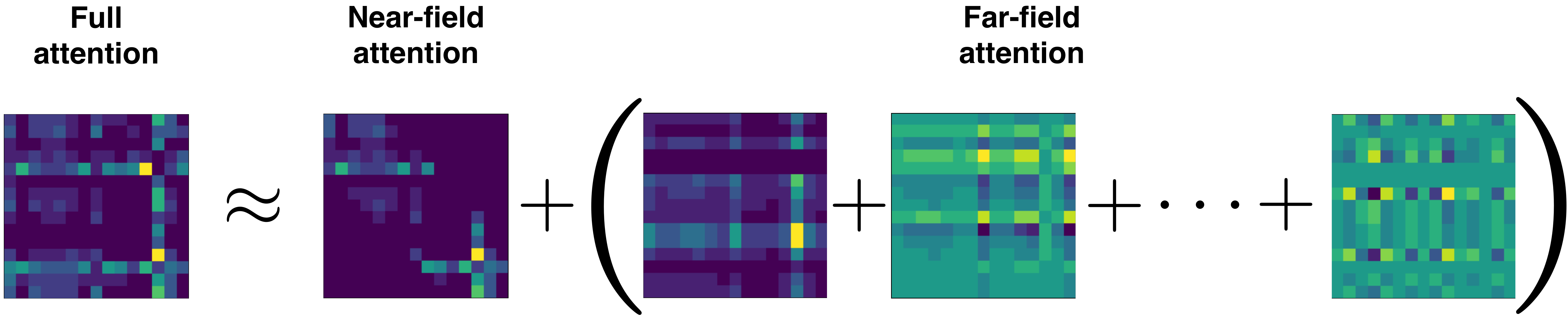}
\vskip -0.1cm
\caption{Left-hand side: we visualize a randomly selected full attention map (the matrix $\mA$ in \eqref{eq:attention-vec}) from the standard transformer trained for the CIFAR10 image classification task in the Long Range Arena (LRA) benchmark. Right-hand side: we illustrate how this attention map can be decomposed into near-field and far-field attention, which are modeled by a sparse banded matrix and the sum of several rank one matrices in our FMMformer, respectively.}
\label{fig:near_far_field}
\end{figure*}

\subsection{Organization}
We structure this paper as follows: In Section~\ref{sec:fmm-attention}, we briefly review the celebrated FMM algorithm and establish the connection between FMM and self-attention. In Section~\ref{sec:fmmformer-implementation}, we present a practical implementation of FMMformers that leverages existing techniques for low-rank matrix approximation. We validate and empirically analyze the efficiency and accuracy of FMMformers in Sections~\ref{sec:exp}. We discuss related works in Section~\ref{subsec:related-work}. The paper ends up with concluding remarks. Technical proofs and more experimental details are provided in the Appendix. 

\section{Fast Multipole Method and Self-attention Mechanism}\label{sec:fmm-attention}
\subsection{Fast multipole method vs. sparse and low-rank matrix approximation}\label{subsec:fmm-Hmatrix}
FMM is a numerical method that was originally developed to speed up the calculation of long-range forces in the $n$-body problem \cite{greengard1987fast} and has been regarded as one of the top $10$ algorithms in scientific computing in the $20$th century~\cite{cipra2000best}. The key idea is that {\em the far-field interaction can be well approximated by separable low-rank matrices} while the near-field interaction can be calculated directly. We use the following simple example to illustrate mathematical reasoning. Without ambiguity, we reuse notations in the previous section and assume:

\noindent (A1) $\mA(i, j) = g(| \vq_i - \vk_j |)$ depends on the distance of two vectors $\vq_i$ and $\vk_j$, where $\mA(i, j)$ is the $(i,j)$-th entry of the matrix $\mA \in \RR^{N\times N}$. 

\noindent (A2) The function $g(s)$ is smooth for $s\neq 0$.

\noindent (A3) The function $g$ satisfies $g(st) = g(s)g(t)$.

One noticeable example in the physical application is $g(| \vq_i - \vk_j |) = 1/| \vq_i - \vk_j |^2$, for which the key vectors $\{\vk_j\}$ are the location of source particles and the query vectors $\{\vq_i\}$ are the location of the target points. Assumption (A3) is not essential, which is presented here for the convenience of proof and can be replaced by other separable forms, e.g., $g(st) = g(s) + g(t)$. 
The near-field and far-field are defined through the distance $| \vq_i - \vk_j |$. 

We now explain the low-rank approximation based on the well-separated condition. For the illustration purpose, we assume the index set $\{1,2,\ldots, N\}$ is partitioned into two groups $\{T_1, T_2\}$. 

\begin{definition}
Group $T_1$ is called well-separated from $T_2$ if there exists a vector $\vk^*$ and a number $\delta \in (0,1)$ such that
$$
| \vk_j - \vk^*| \leq \delta | \vq_i - \vk^*| \quad \forall i\in T_1, j\in T_2.
$$
\end{definition}

The vector $\vk^*$ is a representative vector of $\{ \vk_j, j \in T_2\}$, e.g., the center of vectors in $T_2$. For any $\vq_i, i\in T_1$, it is far away from $\{\vk_j, j\in T_2\}$ and the far-field interaction can be approximated well by a function of $| \vq_i - \vk^*|$.  For example, when calculating the gravitation of a galaxy from the Earth, we can simply treat the galaxy as one single point, although the galaxy may contain hundreds of millions of stars.  

For a matrix $\mA$ and two index sets $I$ and $J$, we use $\mA(I,J)$ to denote the submatrix of $\mA$ with the row index set $I$ and the column index set $J$. 

\begin{lemma}\label{lemma-lowrank-approx}
Let  $\{T_1, T_2\}$ be two well-separated index sets. Assume (A1)-(A3) hold. For any $\varepsilon >0$, the sub-matrix $A(T_1, T_2)$ can be approximated by a rank $p$ matrix to a relative tolerance $\varepsilon>0$ in the sense that: there exists rank $p$ matrices $\mU\in \mathbb R^{|T_1|\times p},\mV\in \mathbb R^{|T_2|\times p}$, with $p\geq C|\log_{\delta} \epsilon|$, such that 
$$
|\mA(i, j) - (\mU\mV^{\top})(i,j) | \leq \epsilon, \quad \forall i\in T_1, j\in T_2.
$$
\end{lemma}

The applicability of the analytic kernel function $g$ was limited to partial differential equations or integral equations where Green’s function satisfying (A1)-(A3). In the application of machine learning, it is hard to verify (A1)-(A3). Instead, we use the definition of diagonal-plus-semi-separable matrices from the book \cite[Definition 1.10]{bebendorf2008hierarchical}. We use MATLAB/Numpy notation ${\rm tril}(\mK,p)$ to denote the lower triangular matrix with zeros above the $p$th-subdiagonal of $\mK$ and similar notation ${\rm triu}(\mK,p)$ for the upper triangular part.

\begin{definition}\cite[Definition 1.10]{bebendorf2008hierarchical}\label{def-ss-matrix} A matrix $\mA\in \mathbb R^{N\times N}$ is called $(p,q)$-semi-separable if there exist matrices $\mU,\mV\in \mathbb R^{N\times p}$ and $\mW,\mZ\in \mathbb R^{N\times q}$ such that 
$$
\mA = {\rm triu}(\mU\mV^{\top},0) + {\rm tril}(\mW\mZ^{\top},1).
$$It is called diagonal-plus-semi-separable if  
$$
\mA = \mD + {\rm triu}(\mU\mV^{\top},1) + {\rm tril}(\mW\mZ^{\top},1).
$$
with some diagonal matrix $\mD$.
\end{definition}

Definition \ref{def-ss-matrix} can be naturally extended to include a banded matrix $\mD$ and sum of several low-rank matrices. 
\begin{wrapfigure}{r}{0.32\textwidth}
\vspace{-0.2cm}
\begin{center}
\begin{tabular}{c}
\includegraphics[width=\linewidth]{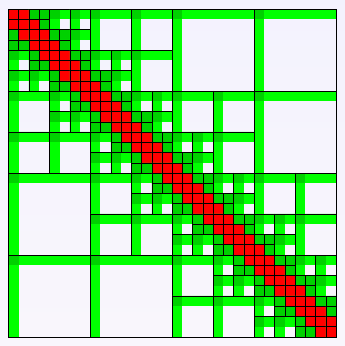}\\
  \end{tabular}
  \end{center}
  \vskip -0.15in
\caption{A $\mathcal H$-matrix based on a hierarchical decomposition of the index set. The red part is a tri-diagonal matrix and the green part can be written as sum of low rank matrices.}\label{fig:Hmatrix}\vspace{-0.15in}
\end{wrapfigure}
One can verify the semi-separable property of matrix $\mK$ by checking the decay of singular values of the matrix. As often used in low-rank approximation methods, the numerical rank or $\varepsilon$-rank of a matrix $K$, for a tolerance $\varepsilon$, is the number of singular values of $\mK$ that are greater than $\varepsilon\|\mK\|_2$. 

In the algebraic counterpart of FMM, the key observation is that off-diagonal matrices are semi-separable. Based on a hierarchical partition of the index set, a $\mathcal H$-matrix \cite{hackbusch1999sparse} can be constructed; see Figure~\ref{fig:Hmatrix} for an illustration. Further compression leads to $\mathcal H^2$-matrix \cite{H2matrix,hackbusch2002data} and the hierarchically semi-separable (HHS) matrix~\cite{chandrasekaran2006fast,xia2010fast}. Other variants include hierarchically block-separable (HBS) \cite{martinsson2005fast}, and hierarchically off-diagonal low-rank (HODLR) \cite{ambikasaran2013mathcal} matrices, etc.

In our application, we write the decomposition as
$$
\mA = \mD + \sum_{l=1}^r \phi_l(\mQ)\phi_l^{\top}(\mK).
$$
In the query and key spaces, the vectors $\vq_i$ and $\vk_j$ may not be well-separated. Then  nonlinear feature maps $\phi_l(\cdot), l=1,\cdots,r$ to higher dimensions can be used to make the mapped datasets more separable.

\begin{figure}[!ht]
\centering
\begin{tabular}{cccc}
\hskip -0.2cm\includegraphics[clip, trim=0.01cm 0.01cm 0.01cm 0.01cm, width=0.235\columnwidth]{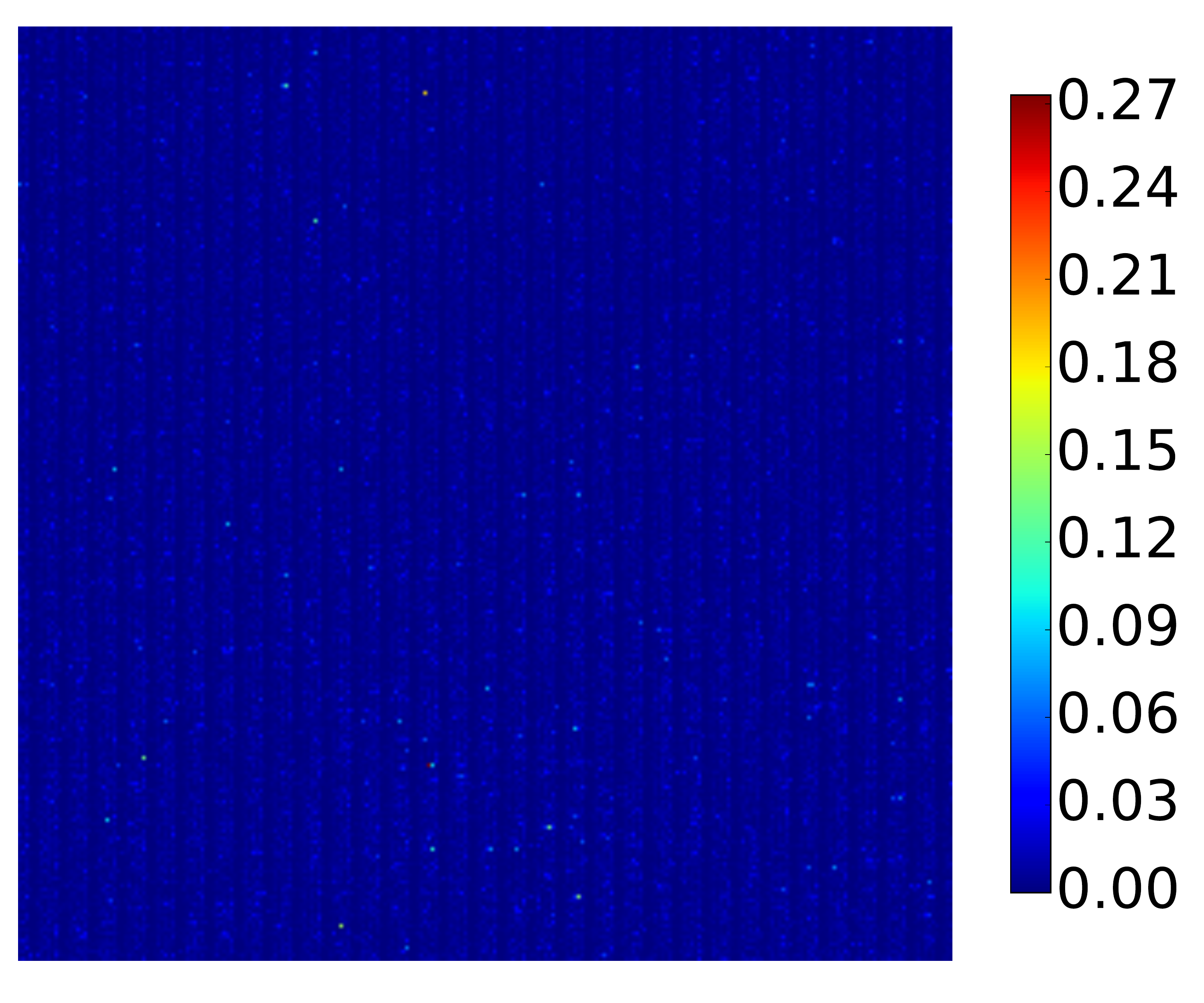}&
\hskip -0.2cm\includegraphics[clip, trim=0.01cm 0.01cm 0.01cm 0.01cm, width=0.235\columnwidth]{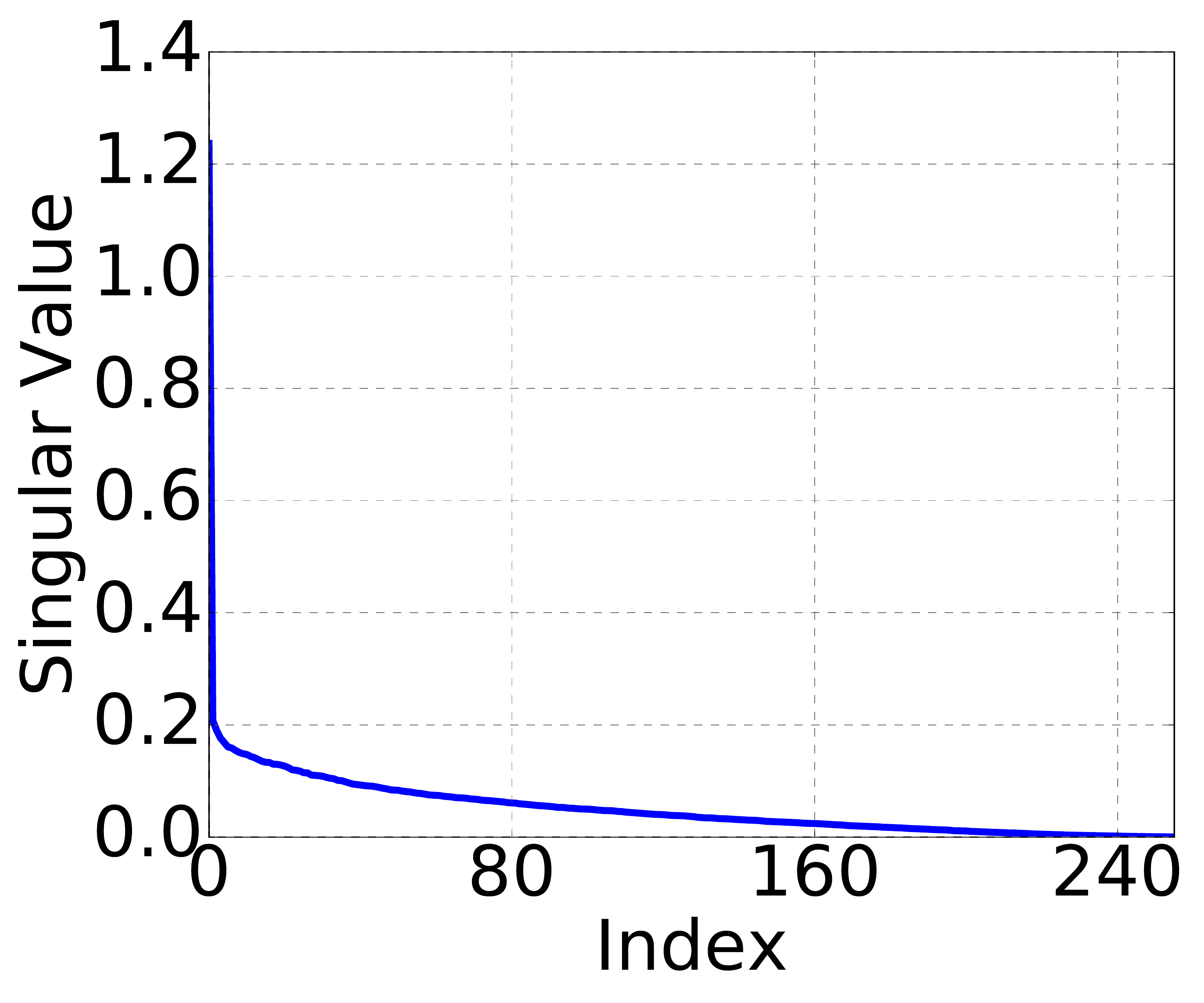}&
\hskip -0.2cm\includegraphics[clip, trim=0.01cm 0.01cm 0.01cm 0.01cm, width=0.235\columnwidth]{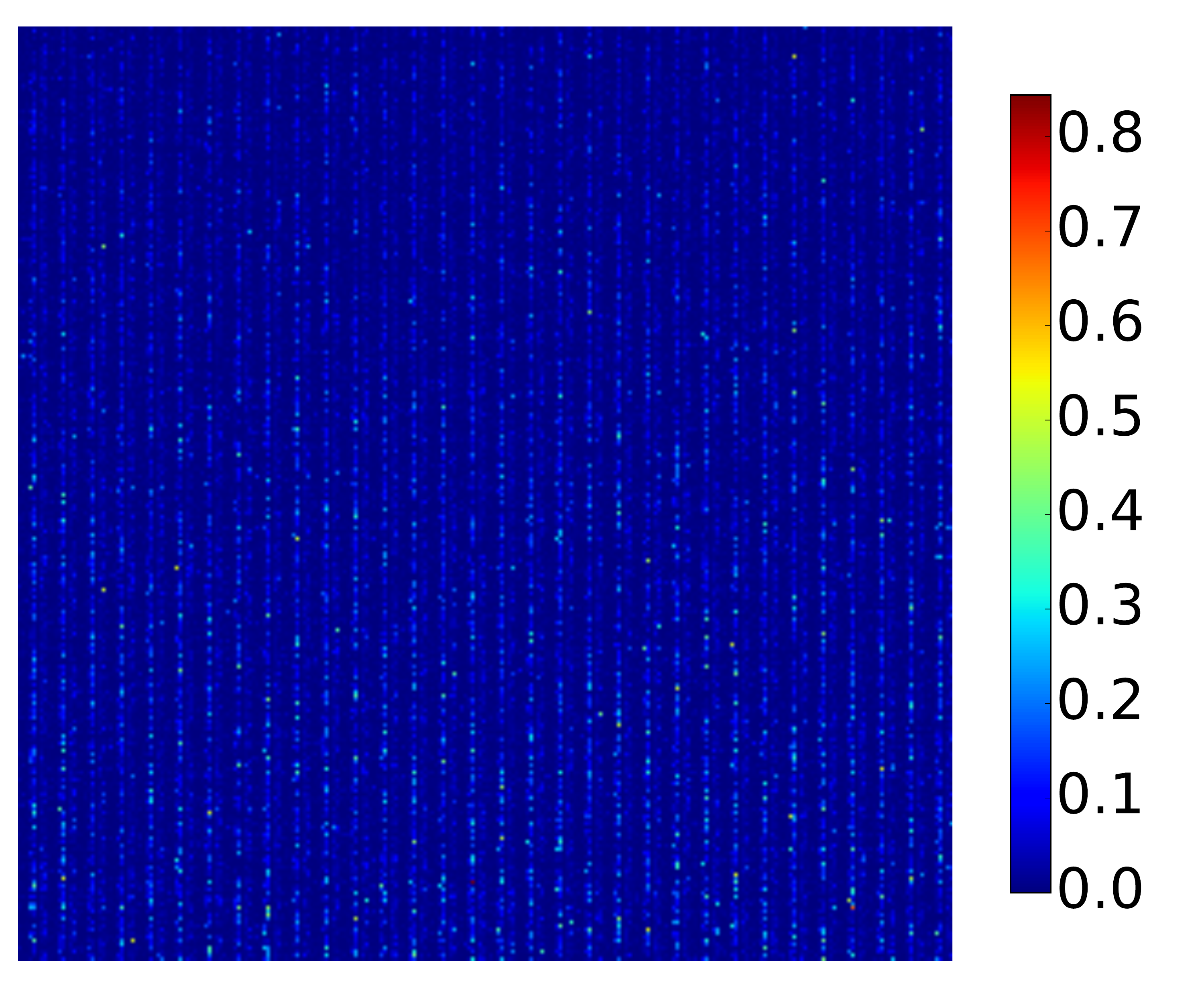}&
\hskip -0.2cm\includegraphics[clip, trim=0.01cm 0.01cm 0.01cm 0.01cm, width=0.235\columnwidth]{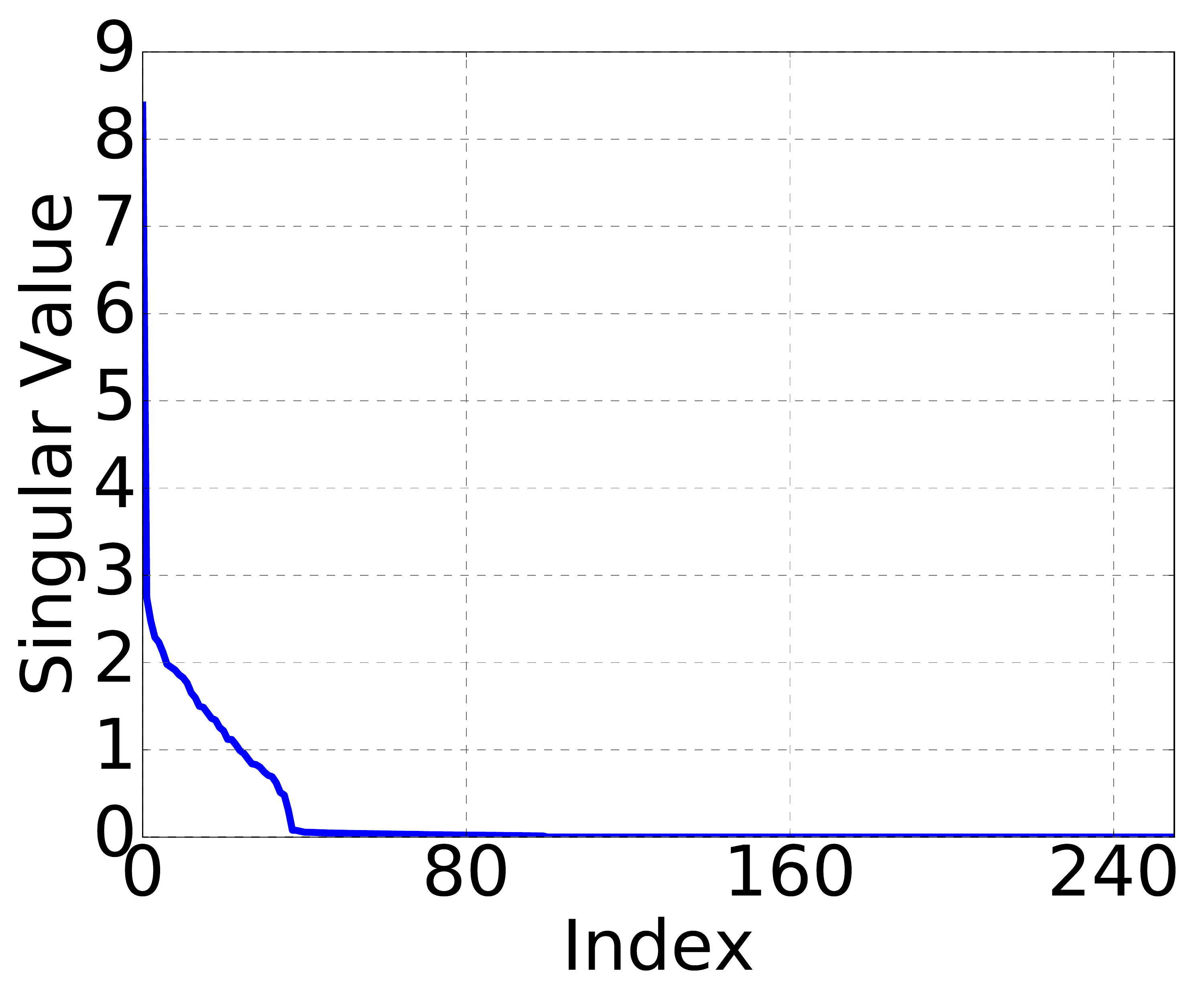}\\
\hskip -0.2cm\includegraphics[clip, trim=0.01cm 0.01cm 0.01cm 0.01cm, width=0.235\columnwidth]{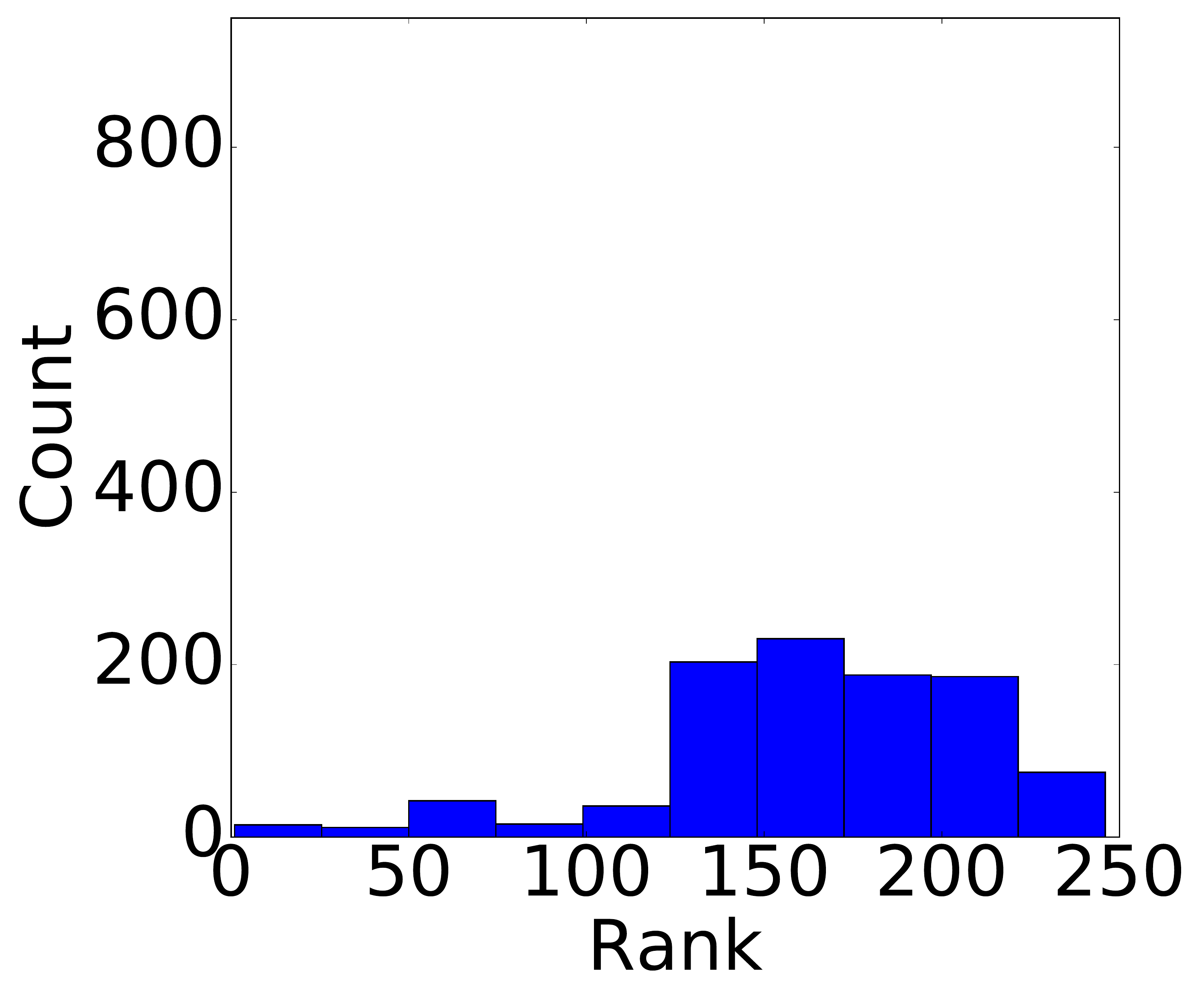}&
\hskip -0.2cm\includegraphics[clip, trim=0.01cm 0.01cm 0.01cm 0.01cm, width=0.235\columnwidth]{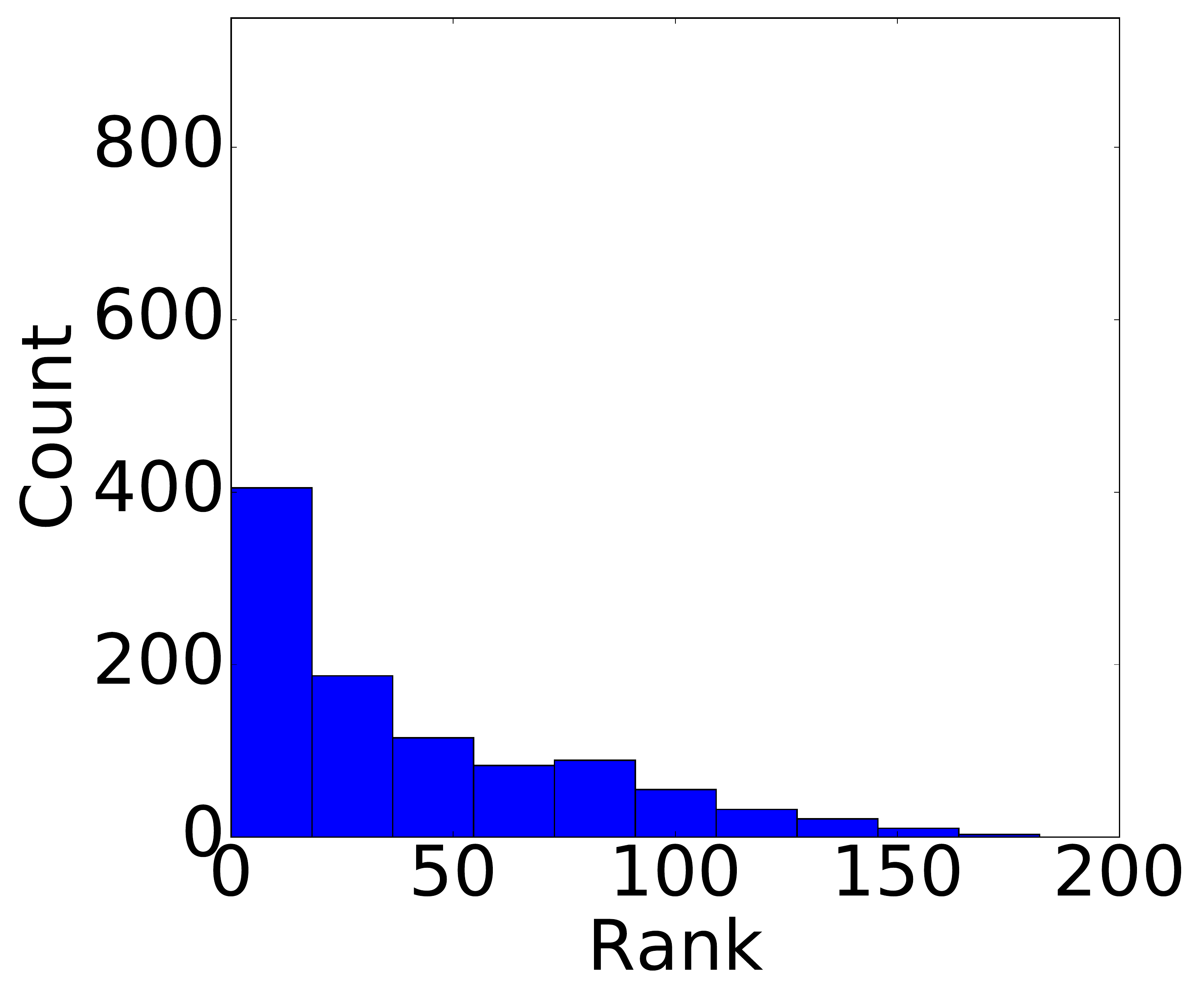}&
\hskip -0.2cm\includegraphics[clip, trim=0.01cm 0.01cm 0.01cm 0.01cm, width=0.235\columnwidth]{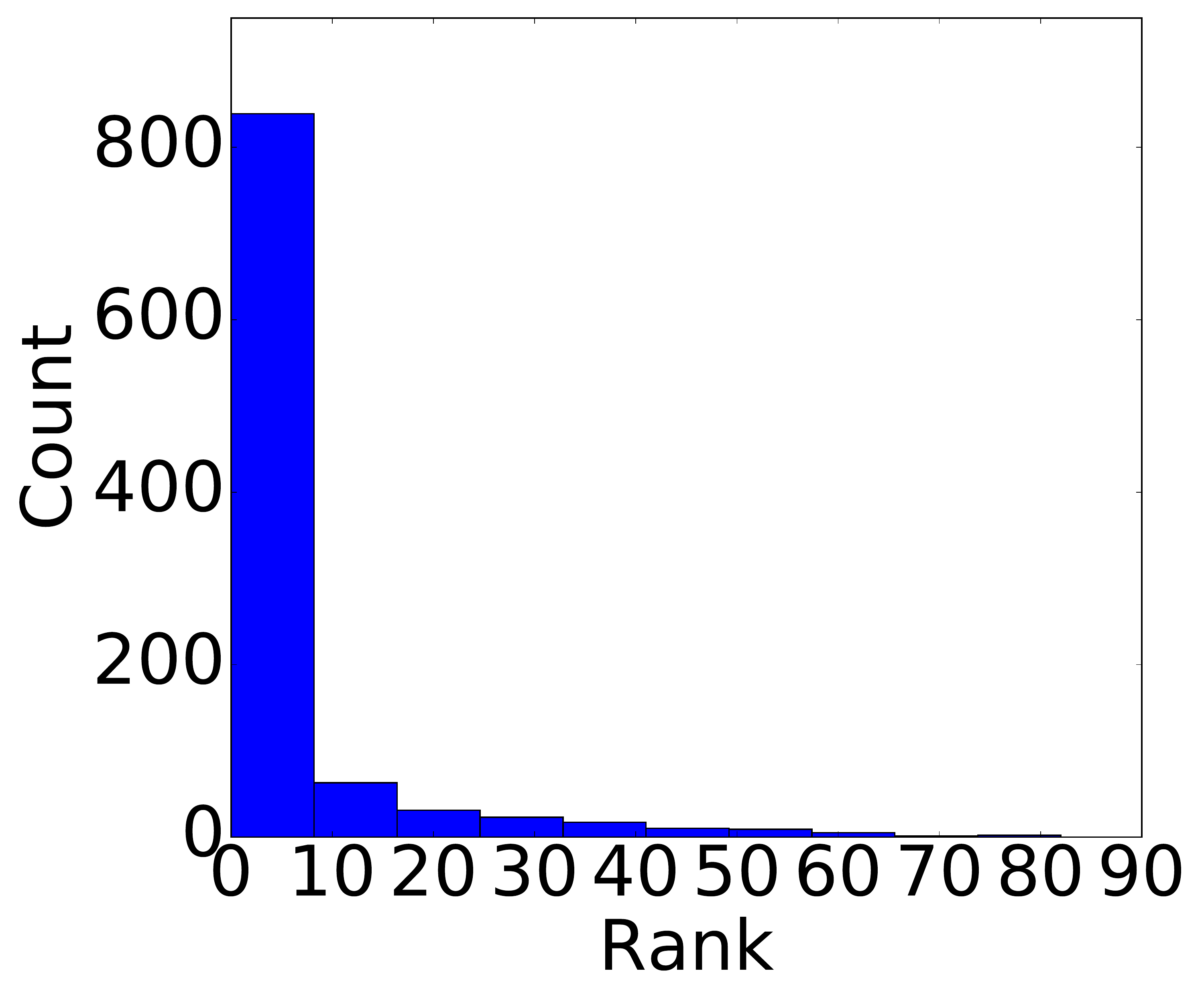}&
\hskip -0.2cm\includegraphics[clip, trim=0.01cm 0.01cm 0.01cm 0.01cm, width=0.235\columnwidth]{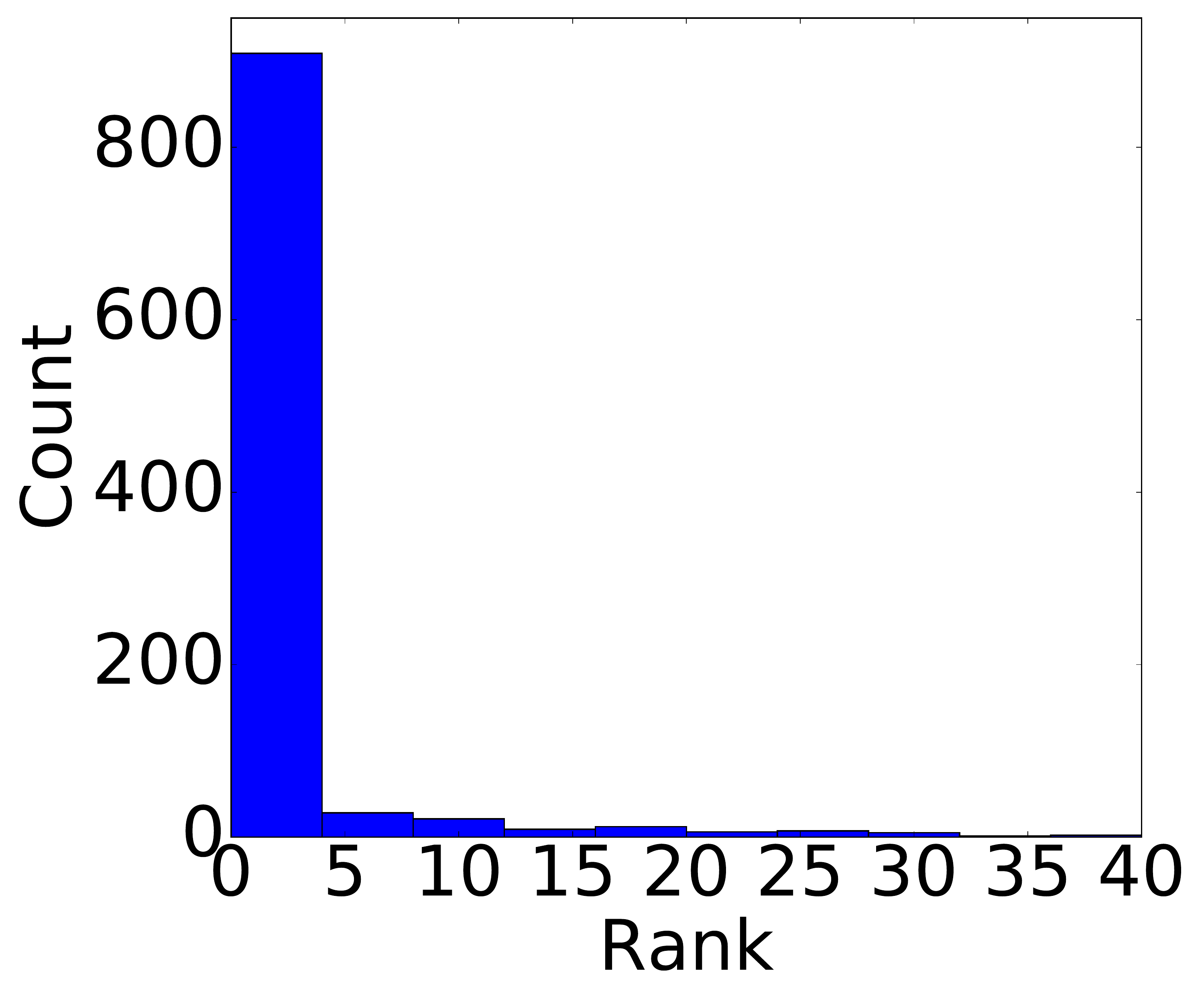}\\
\end{tabular}
\vspace{-3mm}
\caption{First row: plot of two randomly selected attention matrices (left) and their singular values (right) from the transformer trained for WikiText-103 language modeling; see Section~\ref{sec:exp} for details. Second row: distributions of the rank of randomly selected 1000 attention matrices, from the same transformer, after removing a banded matrix $\mD$ of bandwidth 0 (not remove anything from the matrix $\mA$), 5, 10, and 20 (from left to right). Matrix $\mA-\mD$ is of low rank, and the rank becomes smaller in general when the bandwidth of $\mD$ increases.}
\label{fig:low-rank-structure-analysis}
\end{figure}
\subsection{Sparse and low-rank patterns in attention maps}
In this section, we explore the sparse and low-rank structure of the attention matrix $\mA$. In particular, we consider the attention matrix $\mA\in \RR^{256\times 256}$ obtained from the standard transformer trained for WikiText-103 language modeling; see Section~\ref{subsec-WikiText} for the experimental details. We randomly select 1000 different attention matrices, and we exclude a banded matrix $\mD$ with bandwidth 5, 10, and 20 from each of such matrices. Then, we perform singular value decomposition (SVD) to compute the rank of each matrix $\mA-\mD$, and we threshold the small singular values with a magnitude of $10^{-6}$. Figure~\ref{fig:low-rank-structure-analysis} (top row) plots two randomly selected self-attention matrices and the distribution of the rank of the matrix $\mA-\mD$. It is clear that matrix $\mA$ has only a few large singular values and all other singular values are very small. Moreover, matrix $\mA-\mD$ is of low rank, and the rank becomes smaller in general when the bandwidth of $\mD$ increases, which is consistent with the assumptions in Section~\ref{subsec:fmm-Hmatrix}, motivating FMMformers. 

\section{FMMformer: Practical Near-field and Far-field Attention}\label{sec:fmmformer-implementation}
In this section, we present practical algorithms for implementing the proposed FMMformer defined by \eqref{eq:FMMformer-attention}. In particular, we present fast algorithms for computing the near-field attention $\mD\mV$ and the far-field attention $\mL\mV$. 

\subsection{Banded matrix modeling of near-field attention}
We model the near-field attention with the following banded matrix
\begin{equation}\label{eq:banded-k}
\mD = {\rm softmax}\left({\rm band}_k\Big(\frac{{\mQ}{\mK}^\top}{\sqrt{D}}\Big)\right),
\end{equation}
where the operator ${\rm band}_k({*})$ represents taking only the banded part of the matrix $*$ with a bandwidth $k$ ($k\ll N$). In practice, there is no need to calculate the matrix product $\mQ\mK^\top$. Instead, we only need to calculate the products of the vectors that correspond to the nonzero entries of the banded matrix ${\rm band}_k({{\mQ}{\mK}^\top}/{\sqrt{D}})$. Note that for long sequences, both the time and memory complexity of computing \eqref{eq:banded-k} are $\mathcal{O}(N)$.

\subsection{Low-rank matrix modeling of far-field attention}
We consider practical and efficient low-rank matrix modeling of the far-field attention $\mL\mV$ in \eqref{eq:attention-vec}. In principle, any existing off-the-shelf low-rank attention can be integrated into FMMformer to model the far-field attention. In particular, we model the far-field attention leveraging the kernel trick used in ~\cite{katharopoulos2020transformers,performer,schlag2021linear}, which is flexible in selecting different kernels to modulate the rank of the far-field attention component. 

\subsubsection{Low-rank attention via kernelization} 
Suppose we model the far-field attention using a rank $r$ matrix $\mL \in \RR^{N\times N}$, which can be written as the sum of $r$ rank one matrices, i.e.,
\begin{equation}\label{eq:rank-one-decomposition}
\mL = \va_1\vb_1^\top + \va_2\vb_2^\top + \cdots + \va_r\vb_r^\top,
\end{equation}
where $\va_1,\va_2,\cdots,\va_r; \vb_1,\vb_2,\cdots,\vb_r\in \RR^N$. Note that 
\begin{equation}\label{eq:regroup}
\mL\mV = (\va_1\vb_1^\top+\va_2\vb_2^\top+\cdots +\va_r\vb_r^\top)\mV = \va_1(\vb_1^\top\mV) + \va_2(\vb_2^\top\mV) + \cdots + \va_r(\vb_r^\top\mV),
\end{equation}
which indicates that we can compute $\mL\mV$ with $\mathcal{O}(N)$ time complexity using the fact that $\mL\mV=\va_1(\vb_1^\top\mV) + \va_2(\vb_2^\top\mV) + \cdots + \va_r(\vb_r^\top\mV)$.
Also, we only need to store the vectors ${\bf u}_1,{\bf u}_2,\cdots,{\bf u}_r$; ${\bf v}_1,{\bf v}_2,\cdots,{\bf v}_r \in \RR^N$, resulting in linear complexity in memory footprint. 

We borrow the idea of kernelization from the linear transformer \cite{katharopoulos2020transformers} for practical implementation of \eqref{eq:regroup}. In particular, the authors in \cite{katharopoulos2020transformers} generalize the softmax function in \eqref{eq:attention-vec} to a general kernel function $k(\vq_i,\vk_j)$, i.e.,
\begin{equation}\label{eq:attention-kernel}
\underbrace{\hat{\vv}_i=\frac{\sum_{j=1}^N \exp({\vq}_i,{\vk}_j){\vv}_j }{\sum_{j=1}^N \exp({\vq}_i,{\vk}_j)}}_{\mbox{self-attention}}\footnote{Here, $\exp(\vq_i,\vk_j):=\exp(\vq_i^\top\vk_j/\sqrt{D})$.}\Longrightarrow \underbrace{\hat{\vv}_i=\frac{\sum_{j=1}^N k({\vq}_i,{\vk}_j){\vv}_j }{\sum_{j=1}^N k({\vq}_i,{\vk}_j)}}_{\mbox{generalized self-attention}}.
\end{equation}
Under certain assumptions in \cite{mercer1909xvi}, we can linearize the generalized self-attention in \eqref{eq:attention-kernel} as follows,
\begin{equation}\label{eq:attention-kernel-linear}
\hat{\vv}_i=\frac{\sum_{j=1}^N k({\vq}_i,{\vk}_j){\vv}_j }{\sum_{j=1}^N k({\vq}_i,{\vk}_j)}=\frac{\sum_{j=1}^N\phi({\vq}_i)^\top\phi({\vk}_j){\vv}_j}{\sum_{j=1}^N\phi({\vq}_i)^\top\phi({\vk}_j)} = \frac{ \phi({\vq}_i)^\top\sum_{j=1}^N\phi({\vk}_j){\vv}_j^\top }{\phi({\vq}_i)^\top\sum_{j=1}^N\phi({\vk}_j) },
\end{equation}
where $\phi(\cdot)$ is a feature map function. Note that \eqref{eq:attention-kernel-linear} can be regarded as a rank one approximation of self-attention. We can rewrite \eqref{eq:attention-kernel-linear} into the following compact form
\begin{equation}\label{eq:attention-kernel-linear-matrix}
\Hat{\mV}=\frac{\phi({\mQ})(\phi({\mK})^\top{\mV} )}{\phi({\mQ})\phi({\mK})^\top}.
\end{equation}
To generalize \eqref{eq:attention-kernel-linear} to the rank $r$ approximation, we select a set of linearly independent feature maps $\{\phi_l(\cdot)\}_{l=1}^r$. Together with the sparse banded matrix modeling of the near-field attention, we propose the following efficient attention model for the FMMformer
\begin{equation}\label{eq:rank-k-sparse-attention}
\Hat{\mV} = {\mD}{\mV} + \sum_{l=1}^r\frac{\phi_l({\mQ})(\phi_l({\mK})^\top{\mV}) }{\phi_l({\mQ})\phi_l({\mK})^\top}.
\end{equation}
It is evident that both computational time and memory complexity are linear in computing \eqref{eq:rank-k-sparse-attention}. Our design is flexible to selecting feature maps and the sparse banded matrix, which the users can customize. Moreover, causal masking can be implemented easily by truncating the sum from $1$ to $i$ in \eqref{eq:attention-kernel-linear} together with masking out the corresponding part of the banded matrix ${\bf D}$.

\begin{proposition}\label{prop:rank-k}
Let $\phi_l(\vx)\in \RR^N$ ($l=1,2,\cdots,r$ and $r\ll N$) for $\vx\in \RR^n$. If $\{\phi_l(\vx)\}_{l=1}^r$ are linearly independent at $\vx$, then the following matrix ${\mL}(\vx)\in \RR^{N\times N}$ has rank $r$,
\begin{equation}\label{eq:rank-k-prop}
{\mL}(\vx) := \phi_1(\vx)\phi_1(\vx)^\top + \phi_2(\vx)\phi_2(\vx)^\top + \cdots + \phi_r(\vx)\phi_r(\vx)^\top.
\end{equation}
\end{proposition}

\paragraph{Feature map selection.} The feature map selection is crucial for the success of far-field attention modeling. In this work, we adopt the existing successful feature map $\phi_1(\vx):={\rm elu}(\vx)+1$ used in the linear transformer \cite{katharopoulos2020transformers} together with $\phi_2(\vx):={\rm elu}(-\vx)+1$, which is a straightforward modification of $\phi_1(\vx)$. Moreover, we consider the third feature map $\phi_3(\vx):={\rm tanh}(\vx)$. It is easy to check that $\phi_1(\vx),\phi_2(\vx)$, and $\phi_3(\vx)$ are linearly independent for almost all $\vx$. We leave how to design a set of feature maps to optimize the far-field attention modeling as future work.

\subsection{Blending of near-field and far-field attention}
Based on our experiments, adding a learnable weight in front of each attention component benefits training and generalization. As such, we propose the following scheme to blend the near-field attention and far-field attention
\begin{equation}\label{eq:FMMformer-attention-blending}
\Hat{\mV} :=(w_1\mD+w_2\mL)\mV,
\end{equation}
where $w_1$ and $w_2$ are two learnable weights, and we enforce their positivity via a ${\rm sigmoid}$ map.

\section{Experimental Results}\label{sec:exp}
In this section, we numerically verify the efficiency of FMMformers and empirically analyze the effects of near-field and far-field attention on various benchmarks, including synthetic sequence copy (Section~\ref{subsec:copy-task}), Long Range Arena (LRA) (Section~\ref{subsec-LRA}), and language modeling (Section~\ref{subsec-WikiText}). We aim to show that: (i) FMMformers are efficient in both computational time and memory footprint. (ii) Multiple kernels benefit learning of the far-field attention. (iii) Blending near-field attention with far-field attention can boost the performance of linear transformers. Throughout this section, we compare FMMformers with linear transformers ({\bf linear}, $r=1$ in \eqref{eq:rank-k-prop}), standard softmax transformers ({\bf softmax}), and softmax transformers that use a banded attention matrix of bandwidth $k$ ({\bf band$_{k}$}). All experiments are conducted on a server with 4 NVIDIA 3090TI GPUs.

\subsection{Synthetic sequence copy task}\label{subsec:copy-task}
We first consider a synthetic copy task with various sequence lengths, including 128, 256, and 512. In this task, the model has to duplicate a sequence of symbols. Each training and test sample is a sequence of maximum length 128/256/512 with ten different symbols separated by a dedicated separator symbol. We train all transformers for this task using the same setting as in \cite{katharopoulos2020transformers}. 

\begin{figure}
\begin{center}
\begin{tabular}{ccc}
\hskip-0.2cm\includegraphics[width=0.34\columnwidth]{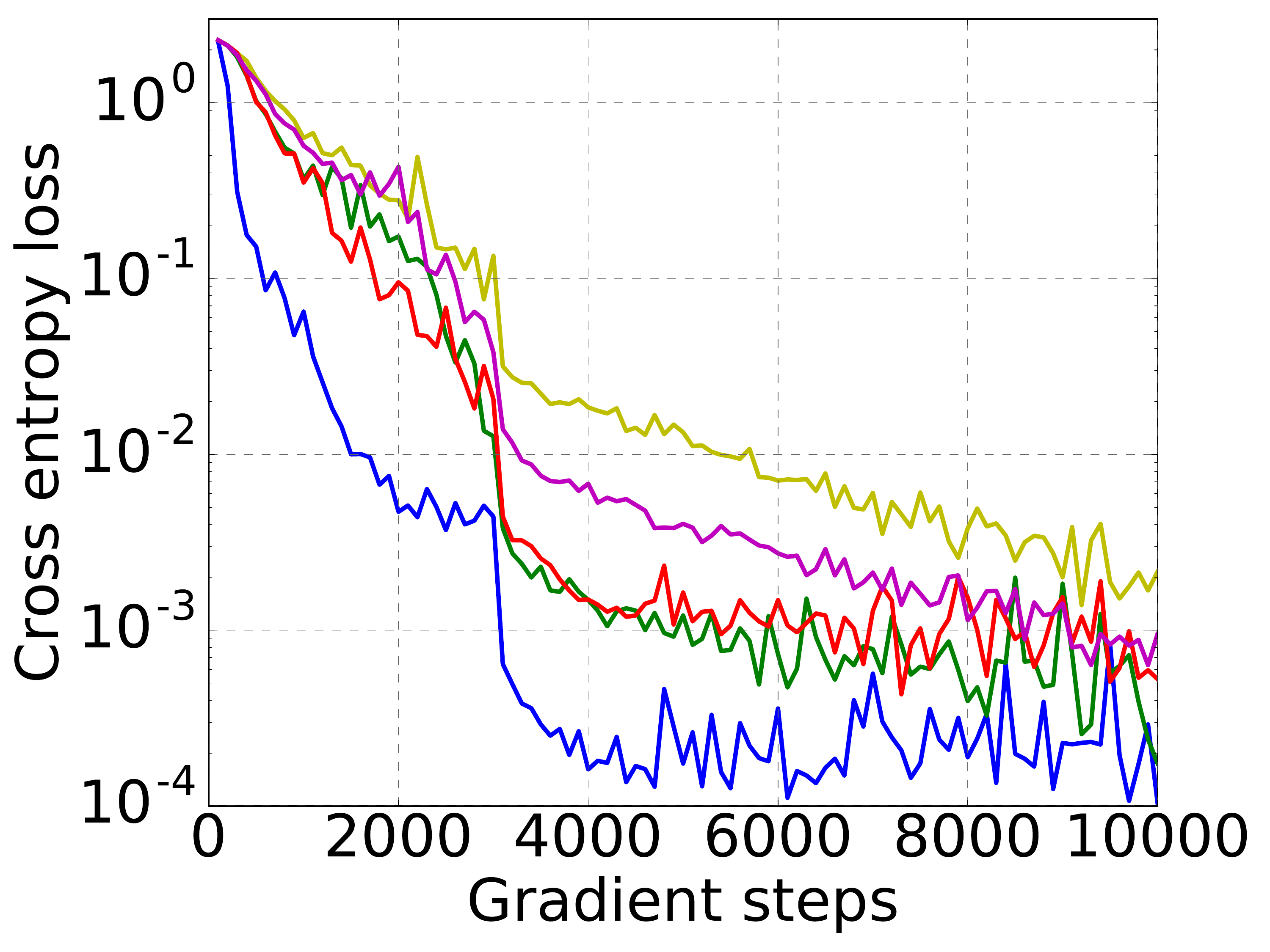} &
\hskip-0.5cm\includegraphics[width=0.34\columnwidth]{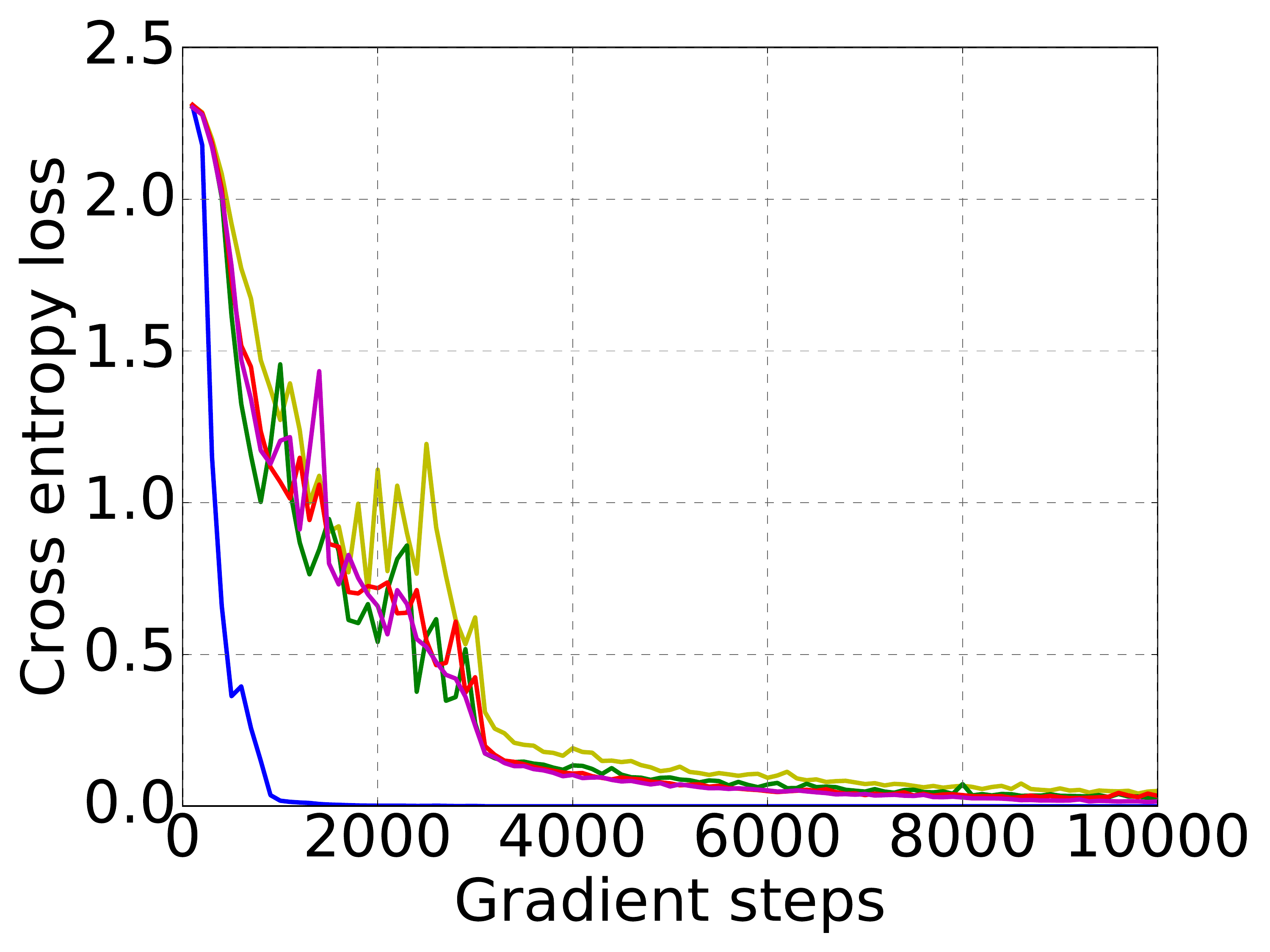}&
\hskip-0.5cm\includegraphics[width=0.34\columnwidth]{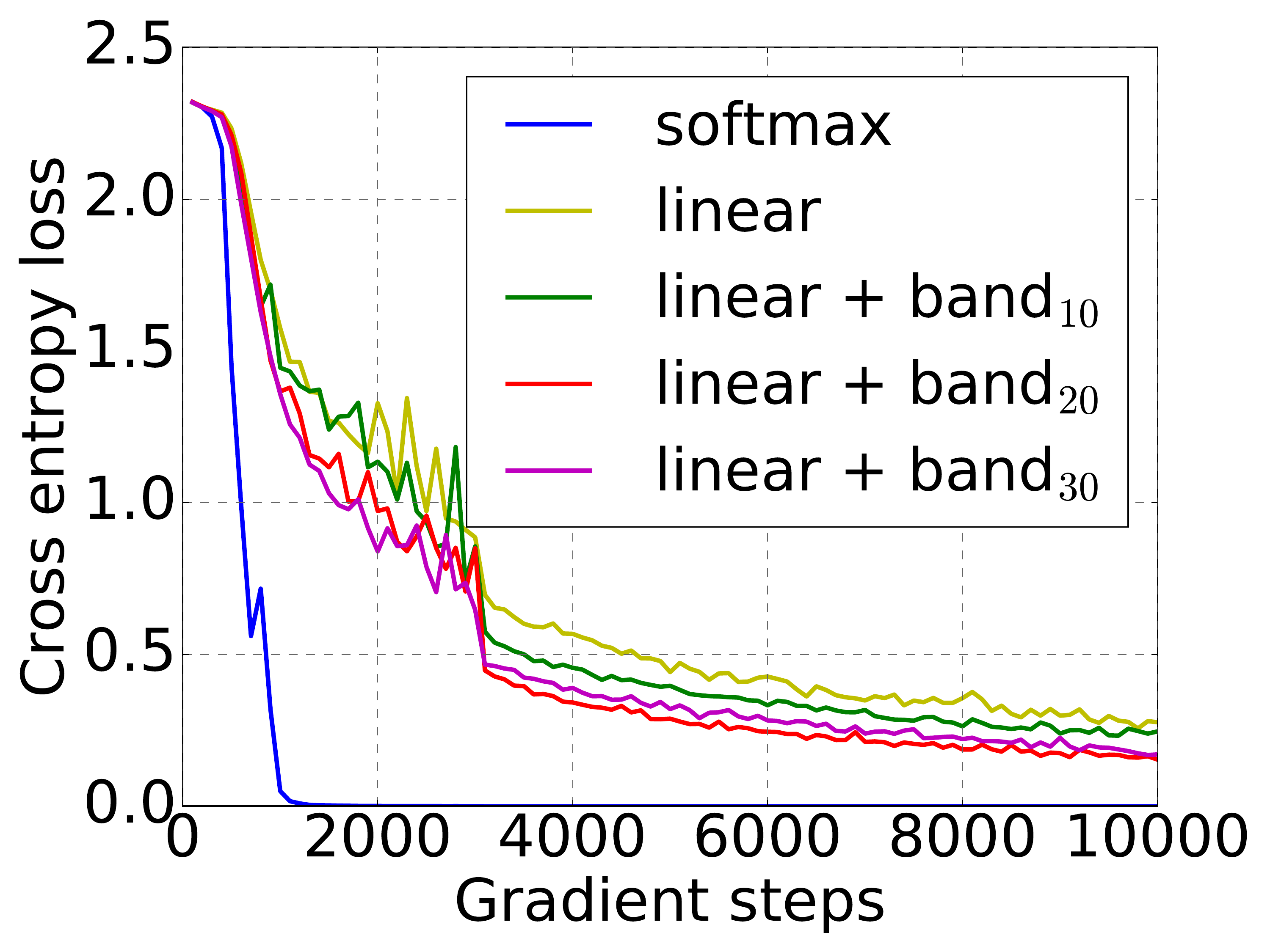}\\
  \end{tabular}
  \end{center}
  \vskip -0.2in
  \caption{ Convergence comparison of softmax, linear, and the blend of linear transformer with a banded matrix on a sequence duplication task with different sequence lengths (left: 128, middle: 256, right: 512). Adding near-field attention into linear attention consistently improves the training for different sequence lengths.
  }\label{fig:copy-task1}
\end{figure}
\paragraph{Boosting performance of linear transformers with near-field attention.}
We first compare FMMformers, obtained by blending the linear transformer with a banded attention matrix of bandwidths 10, 20, and 30, respectively. Figure~\ref{fig:copy-task1} shows that for shorter sequences of length 128, all transformers reach similar loss; the standard softmax transformer converges much faster than the linear transformer while {\em blending the linear transformers with near-field attention can improve training}. Moreover, the benefits of near-field attention become more significant as the sequence length increases. 


\begin{figure}
\begin{center}
\begin{tabular}{ccc}
\hskip-0.2cm\includegraphics[width=0.34\columnwidth]{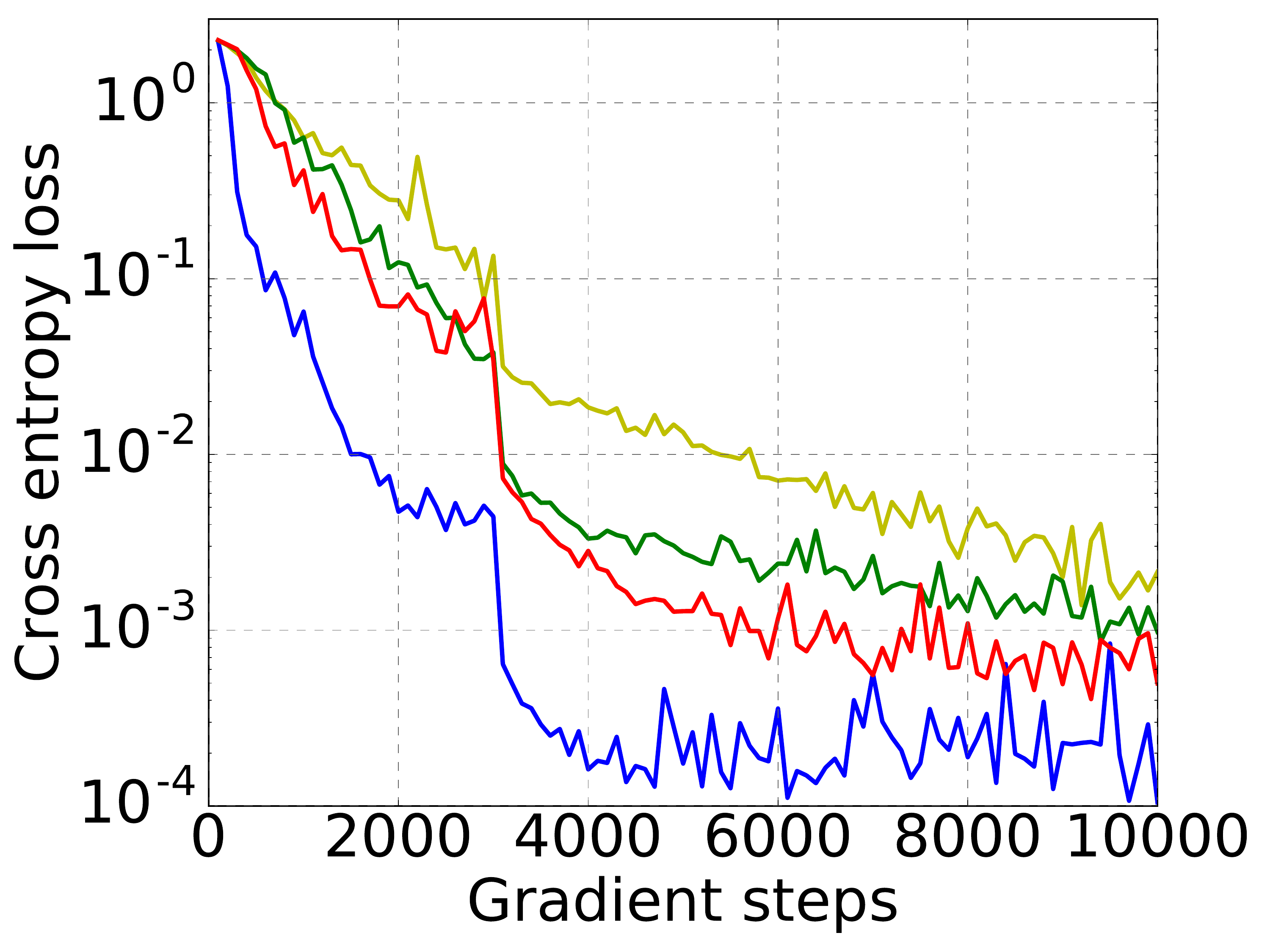} &
\hskip-0.5cm\includegraphics[width=0.34\columnwidth]{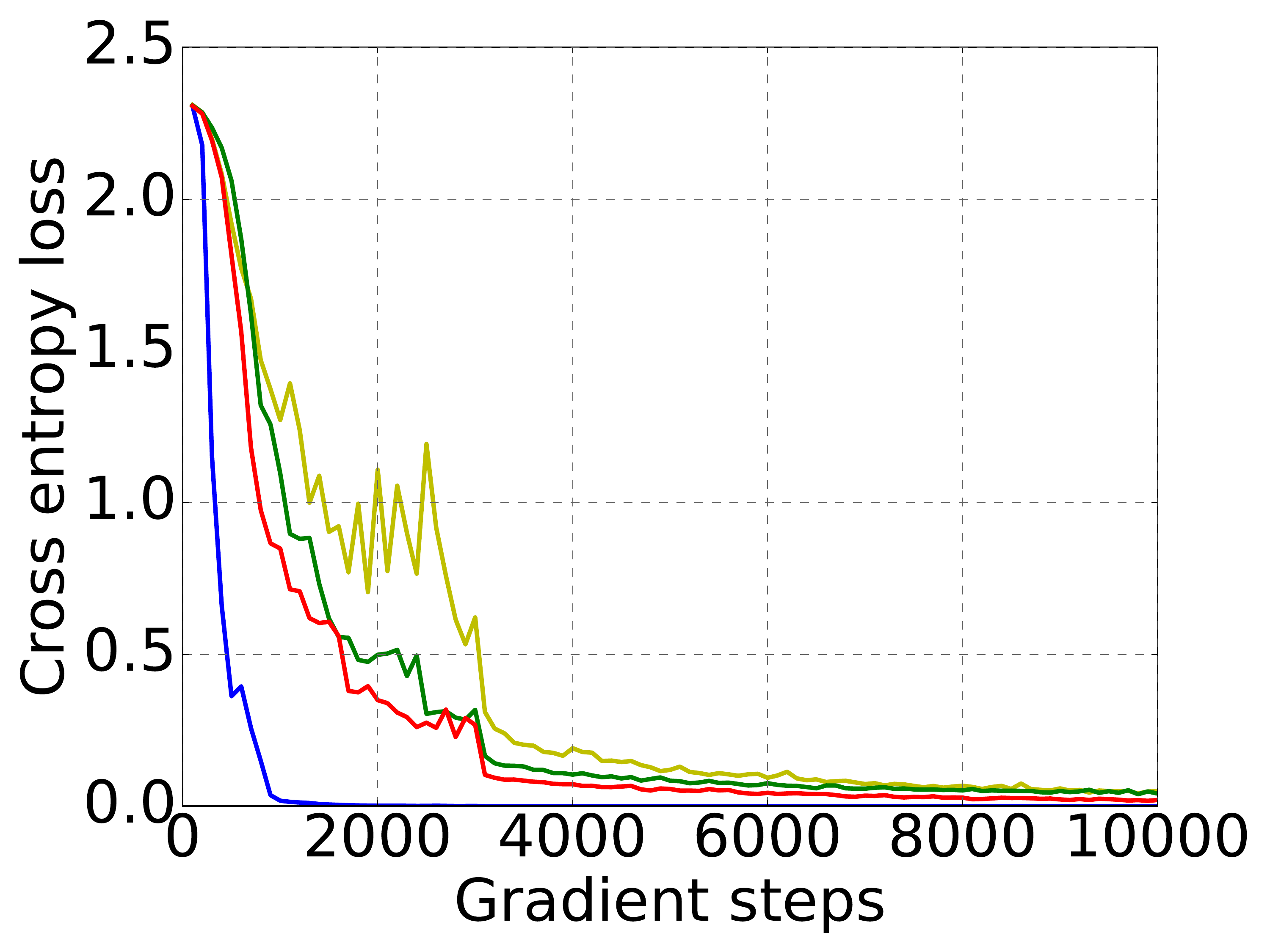}&
\hskip-0.5cm\includegraphics[width=0.34\columnwidth]{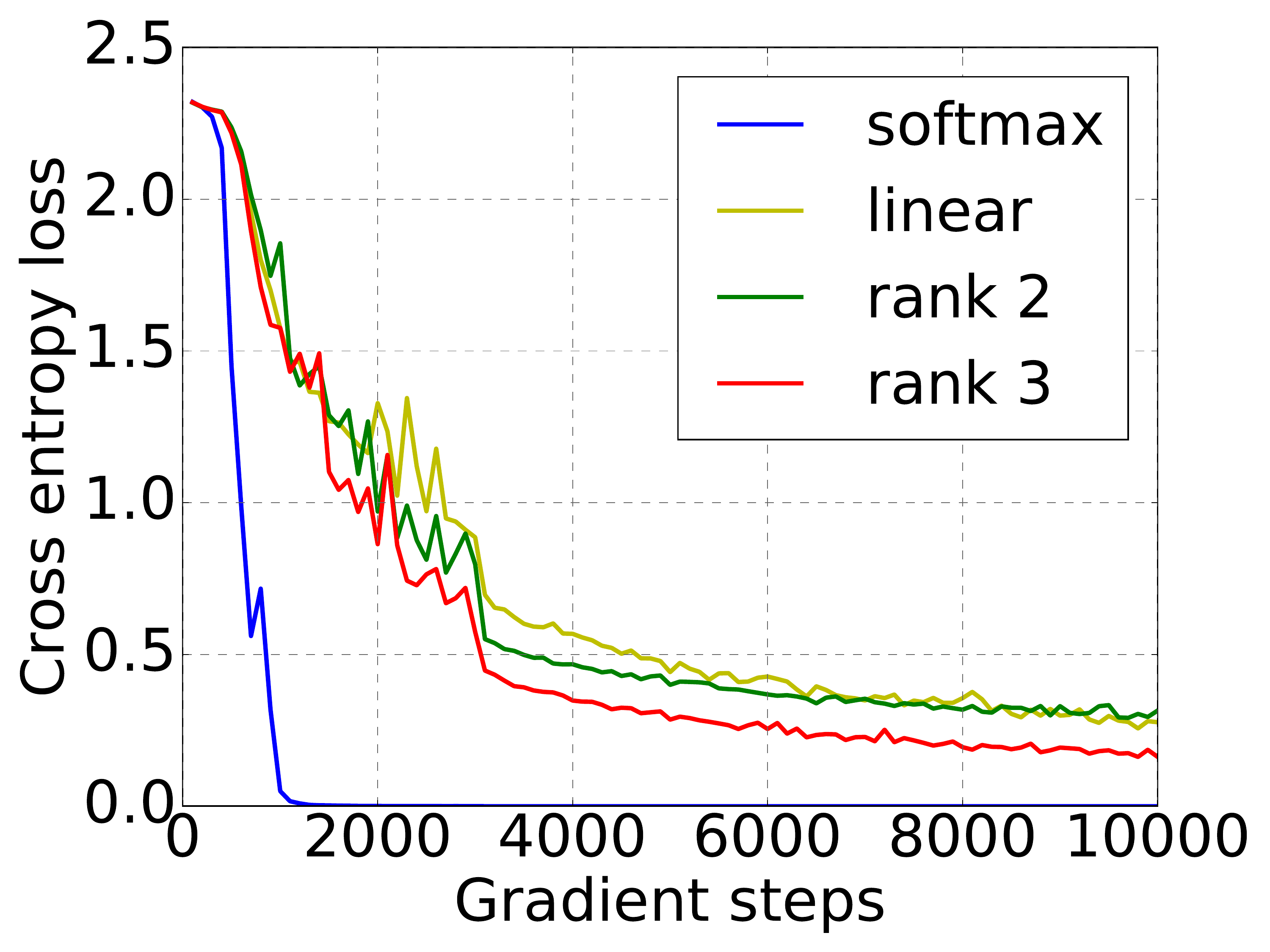}\\
  \end{tabular}
  \end{center}
  \vskip -0.2in
  \caption{ Convergence comparison of softmax, linear, and different low-rank attention on a sequence duplication task with different sequence lengths (left: 128, middle: 256, right: 512). Attention with a higher rank improves training for different sequence lengths.
  }\label{fig:copy-task2}
\end{figure}

\paragraph{Enhancing far-field attention with multi-kernels.} After observing that the linear transformer performs poorly as the sequence length increases, we consider augmenting the linear transformer with multiple feature maps; in particular, we consider the three feature maps mentioned above, i.e., $\phi_1(\vx)={\rm elu}(\vx)+1,\phi_2(\vx)={\rm elu}(-\vx)+1$, and $\phi_3(\vx)={\rm tanh}(\vx)$. Figure~\ref{fig:copy-task2} compares different transformers on different sequence lengths, where {\bf rank 2} consists of the feature maps $\phi_1(\vx)$ and $\phi_2(\vx)$, and {\bf rank 3} consists of all three feature maps. These results show that {\em multiple kernels can improve the learning of far-field attention.}

\paragraph{Computational and memory complexity.} In this part, we compare different transformers in computational time and memory cost. Following \cite{katharopoulos2020transformers}, we compute the attention and gradient for input sequences with different lengths $N\in \{2^9,2^{10},\cdots,2^{16}\}$ and measure the peak allocated GPU memory and the required time for each transformer model. We conduct this experiment on an NVIDIA 3090TI with 24GB memory, and we report the time and memory cost per sample in the same way as in \cite{katharopoulos2020transformers}. Figure \ref{fig:copy-task3} contrasts the time (left) and memory (right) costs of different models. 

\begin{figure}
\begin{center}
\begin{tabular}{cc}
\includegraphics[width=0.34\columnwidth]{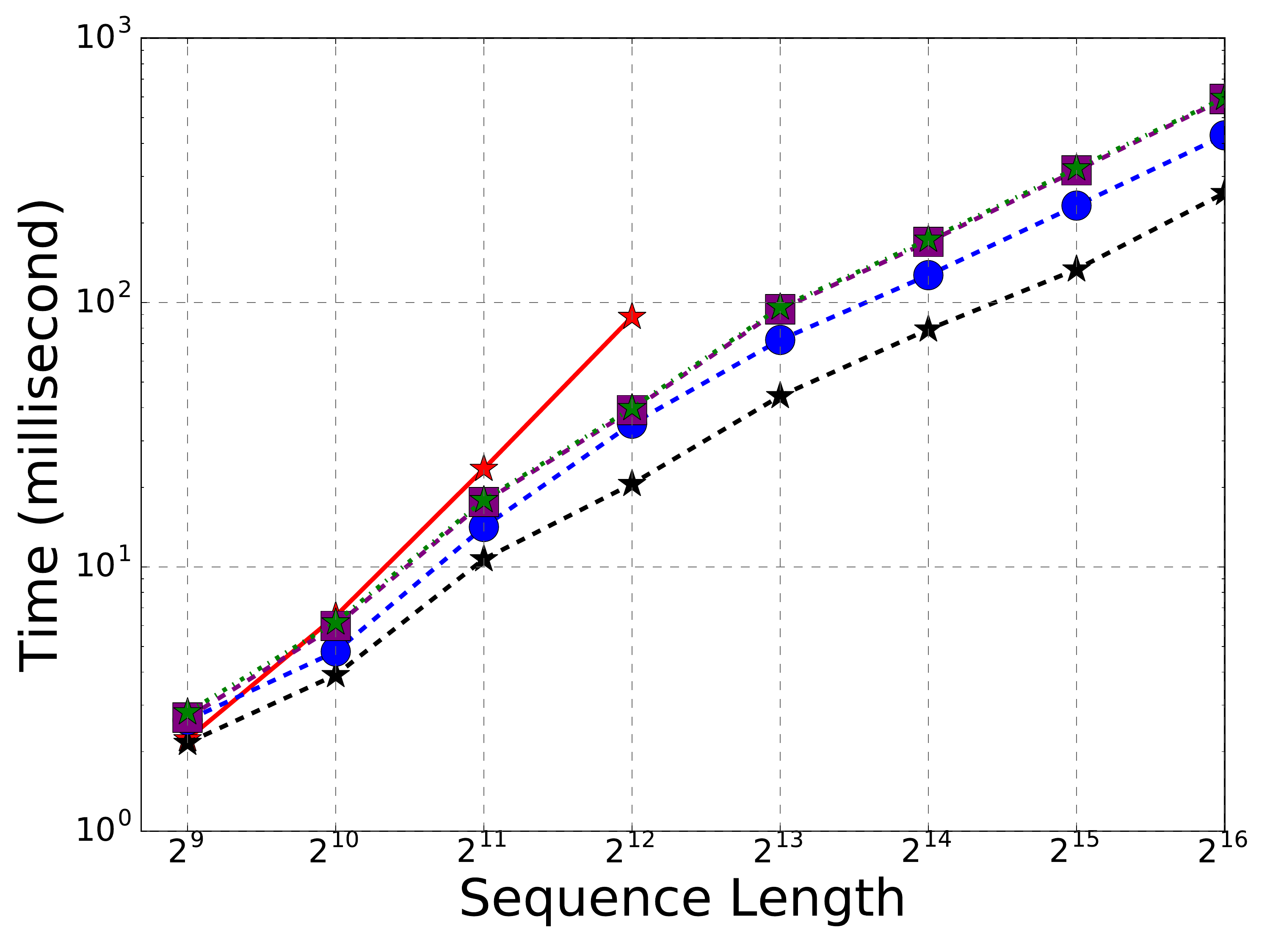} &
\includegraphics[width=0.34\columnwidth]{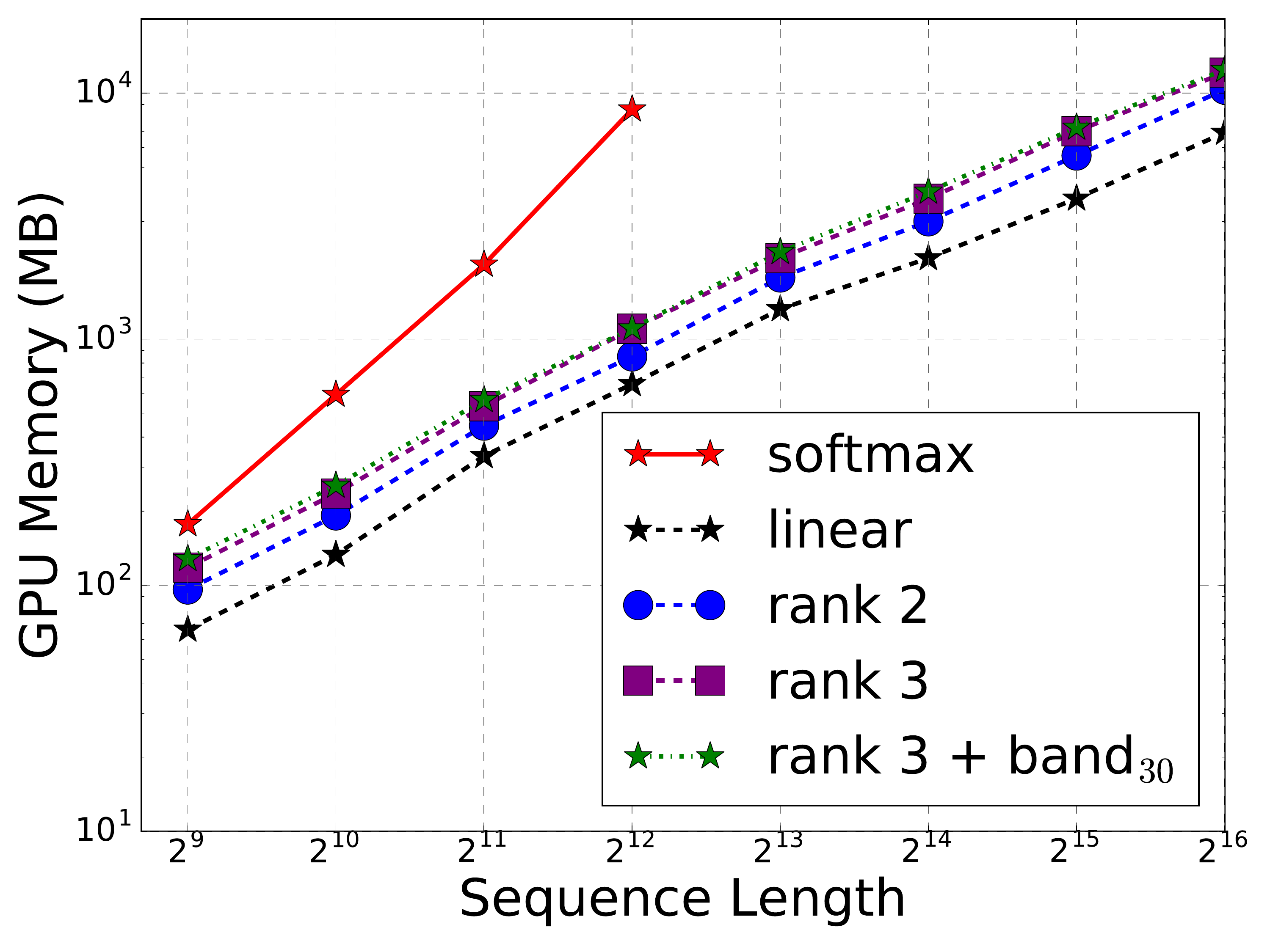}\\
  \end{tabular}
  \end{center}
  \vskip -0.2in
  \caption{ Comparison of the computational time and the peak memory cost of a forward/backward pass for standard softmax transformer, linear transformer, rank 2 linear transformer, rank 3 linear transformer, and the blend of rank 3 linear transformer with a banded attention matrix of bandwidth 30. All transformers are of linear complexity in time and memory except the softmax transformer.
  }\label{fig:copy-task3}
  \vspace*{-0.5em}
\end{figure}

\subsection{Long Range Arena (LRA) Benchmark}\label{subsec-LRA}
In this experiment, we evaluate our model on tasks that involve longer sequence lengths in the Long Range Arena benchmark~\cite{tay2021long}. We show that the FMMformer outperforms the baseline linear transformer and standard softmax transformer~\cite{vaswani2017attention}, justifying the advantage of the FMMformer in capturing long-term dependency. We provide model and training details in the Appendix.

{\bf Datasets and metrics.} We consider all five tasks in the LRA benchmark, including Listops~\cite{nangia-bowman-2018-listops}, byte-level IMDb reviews text classification~\cite{maas-etal-2011-learning}, byte-level document retrieval~\cite{radev2013acl}, CIFAR-10 image classification on sequences of pixels~\cite{krizhevsky2009learning}, and Pathfinder~\cite{linsley2018learn}. These tasks involve long sequences of length $2K$, $4K$, $4K$, $1K$, and $1K$, respectively. We follow the setup/evaluation protocol in \cite{tay2021long} and report the test accuracy for individual task and the average result across all tasks.  

{\bf Results.} We summarize our results in Table~\ref{tab:lra}. Like in the copy task, we observe that adding near-field attention modeled by banded attention matrices improves the performance of linear transformers. More interestingly, using bandwidth 5 already yields good results across all LRA tasks while significantly reducing the computational and memory cost of calculating the attention matrix. For example, in the byte-level document retrieval~\cite{radev2013acl} task, a banded matrix with bandwidth 5 only accounts for 0.125\% of the corresponding full attention matrix. The FMMformer with 1 kernel (blending a banded matrix of bandwidth 5 with the linear transformer using feature map $\phi_1(\vx)$) outperforms the linear transformer and yields similar or better results than the standard softmax transformer in all tasks. Furthermore, the FMMformer with 2 kernels (blending a banded attention matrix of bandwidth 5 with the linear transformer using feature maps $\phi_1(\vx)$ and $\phi_2(\vx)$) further improves the FMMformer with 1 kernel, justifying the need of better low-rank approximation for the far-field attention. Across tasks, the FMMformer obtains the best average accuracy. Also, it is worth noting that tasks in the LRA benchmark cover different data modalities include text and images. Good performance of the FMMformer on these tasks demonstrates that the advantages of our model over the linear and standard transformers are consistent across data modalities.

\begin{table*}[!t]
\centering
\resizebox{1.0\linewidth}{!}{
\begin{tabular}{c|c|c|c|c|c|c}
\hline
Model          & ListOps (2K) & Text (4K) & Retrieval (4K) & Image (1K) &  Pathfinder (1K) & Avg \\ \hline\hline
Softmax~\cite{vaswani2017attention} &\bf 37.10 (37.10) & 64.17 (65.02) & 80.71 (79.35) & 39.06 (38.20) & \bf 72.48 (74.16) & 58.70 (58.77) \\
\hline\hline
Linear~\cite{katharopoulos2020transformers}  & 18.30 & 64.22 & 81.37 & 38.29 & 71.17 & 54.67 \\
\hline\hline
Band$_5$ & 32.16 & 66.31 & 79.41 & 43.33 & 67.44 & 57.73 \\
\hline\hline
FMMformer (1-kernel + Band$_5$) & 33.22 & 66.52 & 81.50 & 45.01 & 71.29 & 59.51 \\
FMMformer (2-kernel + Band$_5$) &  36.74 & \bf 67.84 & \bf 81.88 & \bf 45.10 & 72.12 & \bf 60.74 \\
\hline
\end{tabular}}
\hspace{0.1em}
\vspace*{-0.5em}
\caption{Results on the LRA benchmark. We report the test classification accuracy for each task and average accuracy across all tasks. The FMMformer outperforms the linear transformer and attains similar or better results than the standard transformer. Across tasks, the FMMformer achieves the best average accuracy. Also, the FMMformer with 2 kernels enhances the performance of the FMMformer with 1 kernel. The numbers in the parenthesis are from the paper \cite{xiong2021nystromformer}. Note that we use near-field attentions of bandwidth 5 for all FMMformers reported here, and Band$_5$ are softmax transformers with a banded attention matrix of bandwidth 5.
}\label{tab:lra}
\vspace*{-1.0em}
\end{table*}

\subsection{Language Modeling on WikiText-103} 
\label{subsec-WikiText}
Experiments on the copy task in Section~\ref{subsec:copy-task} illustrate the effect of combining near-field and far-field attention. Results on the LRA benchmark in Section~\ref{subsec-LRA} show the ability of our FMMformer to capture very long-term dependency and extend to different data modalities. Now our goal is to confirm the advantage of the FMMformer on a large-scale application. We consider the word-level language modeling task on WikiText-103~\cite{DBLP:conf/iclr/MerityX0S17}.

{\bf Datasets and metrics.} WikiText-103 consists of articles from Wikipedia  and is a dataset with long contextual dependencies. The training set is made up of about $28K$ articles containing $103M$ running words; this corresponds to text blocks of about 3600 words. The validation and test sets are composed of $218K$ and $246K$ running words, respectively. Each of them contains $60$ articles and about $268K$ words. Our experiment follows the standard setting~\cite{DBLP:conf/iclr/MerityX0S17, schlag2021linear} and split the training data into $L$-word independent long segments. For evaluation, we use a batch size of 1, and go through the text sequence with a sliding window of size $L$. We consider only the last position for computing perplexity (PPL) except in the first segment, where all positions are evaluated as in~\cite{al2019character, schlag2021linear}. 

{\bf Results.} Table~\ref{tab:wikitext103} shows the validation and test perplexity of our models versus the linear and standard softmax transformer on WikiText-103. Consistent with previous experiments, the FMMformer outperforms the linear transformer. The standard softmax transformer obtains the best results in this task, but the gap between the FMMformer and the standard transformer is reduced when a larger bandwidth is used for near-field attention in the FMMformer. This is justified by the improvement in terms of PPL of the FMMformer with a near-field attention of bandwidth 20 compared to the FMMformer with a near-field attention of bandwidth 5. Also, FMMformer with 2 kernels ($\phi_1(\vx)$ and $\phi_2(\vx)$) still improves over FMMformer with 1 kernel ($\phi_1(\vx)$). Consider the linear complexity of computational time and memory advantage of FMMformers, the small performance gap of FMMformers to standard softmax transformers can potentially be overcome by using the near-field attention of larger bandwidth and employing more kernels to better capture the far-field attention. 

\bgroup
\renewcommand{\arraystretch}{1.1}
\begin{table}[t!]
    \begin{center}
    \begin{tabular}{ccc}
    \hline
        Method & Valid PPL & Test PPL \\
        \hline
        Softmax~\cite{vaswani2017attention} & 33.15 & 34.29 \\
        \hline\hline
        Linear~\cite{katharopoulos2020transformers}   & 37.27 & 38.40  \\
        \hline\hline
        Band$_5$ & 43.77  & 44.76  \\
        Band$_{20}$ & 38.18 & 39.19 \\
        \hline\hline
        FMMformer (1-kernel + Band$_5$) & 36.27  & 37.29   \\
        FMMformer (1-kernel + Band$_{20}$) & 35.41  & 36.43   \\
        \hline\hline
        FMMformer (2-kernel + Band$_{20}$) & 35.10  & 36.11  \\
        \hline
    \end{tabular}
    \end{center}
    \caption{WikiText-103 language model perplexities of FMMformers compared to the baselines. The number of parameters (40\,M) is almost the same for all models, up to the small difference introduced by additional weights on the far-field attention in FMMformers. FMMformers outperform linear transformers~\cite{katharopoulos2020transformers}. The performance gap compared to softmax transformers is reduced when using a larger bandwidth in near-field attention and more kernels in far-field attention. Note that Band$_5$ and Band$_{20}$ are softmax transformers with a banded attention matrix of bandwidth 5 and 20, respectively.}
    \label{tab:wikitext103}
\vspace*{-0.6em}
\end{table}

\section{Related works.}\label{subsec:related-work}

\paragraph{Low-rank transformers.} Low-rank approximation of the self-attention matrix $\mA$ has been a popular method in reducing the quadratic computational and memory complexity of transformers to linear. Linearized attention that leverages kernelization tricks can be considered as the rank one approximation of the self-attention matrix \cite{wang2020linformer,katharopoulos2020transformers,performer,shen2021efficient}; the choice of the feature map function is crucial for the success of linearized attention. Fast weight memories \cite{schlag2021learning} have been used to improve memory capacity of linearized attention \cite{schlag2021linear}. The Nystr\"om method has also been leveraged for developing efficient attention with linear computational complexity \cite{xiong2021nystromformer}. Many other low-rank attention models exist, e.g., \cite{pmlr-v80-blanc18a,NEURIPS2019_e43739bb,song2021implicit,peng2021random}. FMMformers employ low-rank attention to model far-field attention; in principle, the merits of existing low-rank attention can be integrated into FMMformers.

\paragraph{Sparse transformers.} 
Attention matrices have been enforced with different sparsity patterns to gain efficiency, including fixed sparsity patterns \cite{qiu2019blockwise,pmlr-v80-parmar18a,beltagy2020longformer,ainslie-etal-2020-etc,zaheer2021big}, a combination of different sparsity patterns \cite{child2019generating,ho2019axial}, and data-dependent/learnable sparsity patterns \cite{j.2018generating,wang2020linformer,tay2020synthesizer,pmlr-v119-tay20a,Kitaev2020Reformer,roy-etal-2021-efficient,vyas2020fast}. Informer \cite{haoyietal-informer-2021} is another efficient attention scheme using a sparse query. Note that the existing sparsity pattern can be very complicated, while we adopt a sparse banded matrix to model the near-field attention.

\paragraph{Other efficient transformers.} 
Reformer reduces the computational cost of self-attention to $\mathcal{O}(N\log N)$ via locality-sensitive hashing \cite{Kitaev2020Reformer}. Linformer \cite{wang2020linformer} and Longformer \cite{beltagy2020longformer} obtain the linear complexity using random projection and local window attention, respectively. FNet \cite{leethorp2021fnet} replaces the expensive attention with a simple Fourier transformer-based token mixing scheme. The lambda layer \cite{bello2021lambdanetworks} also belongs to efficient attention. See \cite{tay2020efficient} for a recent review on efficient attention. 

\paragraph{Sparse and low-rank interpretation of FMM.}
The fast multipole method is introduced by Greengard and Rokhlin \cite{greengard1987fast} for the efficient computation of gravitational/electrostatic potentials and fields. Applications of FMM to machine learning can be found in \cite{gray2001n,lee2012distributed,van2014accelerating}. FMM can be generalized algebraically for efficient computation of the dense matrix-vector product. Algebraic counterpart of FMM include $\mathcal H$-matrix \cite{hackbusch1999sparse}, $\mathcal H^2$-matrix \cite{H2matrix,hackbusch2002data}, hierarchically semi-separable (HSS) \cite{chandrasekaran2006fast,xia2010fast}, hierarchically block-separable (HBS) \cite{martinsson2005fast}, and hierarchically off-diagonal low-rank (HODLR) \cite{ambikasaran2013mathcal} matrices. The common feature is to compress the off-diagonal sub-matrices by low-rank approximations.

\section{Concluding Remarks}\label{sec:conclusion}
In this paper, we proposed FMMformers, a class of efficient and flexible transformers with linear time and memory complexity, inspired by the fast multipole method. In FMMformers, we decompose the full attention into near-field and far-field attention; we model the near-field attention with a sparse banded matrix and model the far-field attention using a low-rank matrix leveraging ideas of the linear transformer \cite{katharopoulos2020transformers}. We validate the efficiency of FMMformers on various benchmark tasks, including synthetic sequence copy, LRA benchmark, and WikiText-103 language modeling. Our numerical results show that FMMformers consistently outperform the linear transformer on all benchmarks and outperform the standard softmax transformer on the LRA tasks. In our work, we select linearly independent feature maps to enhance the learning of far-field attention. It is natural to ask how to design a set of feature maps to optimize the performance of FMMformers? Furthermore, we leave the application of FMMformers for improving the vision transformer \cite{dosovitskiy2020image,touvron2020deit} as future work.

\clearpage


\clearpage

\section{Proof of Lemma~\ref{lemma-lowrank-approx}}
\begin{proof}[Proof of Lemma~\ref{lemma-lowrank-approx}]
For simplicity, we consider the scalar case, i.e., $\vq_i, \vk_j\in \mathbb R$. The vector case can be handled by switching to the polar coordinate and expand in the spherical harmonic basis. 

Let $r = \vk_j - \vk^*/|\vq_i - \vk^*|$. Then $|\vq_i - \vk_j| = |\vq_i - \vk^*|(1+ r)$ and $ |r|\leq \delta$ for $i\in T_1, j\in T_2$. Then
\begin{align*}
\mA(i, j) 
&= g( |\vq_i - \vk^*| ) \, g( 1+ r )\\
&= g(|\vq_i - \vk^*| )\sum _{m=0}^p\frac{g^{(m)}(1)}{m!}\frac{(\vk_j - \vk^*)^m}{|\vq_i - \vk^*|^m} + \mathcal O(\delta ^{p+1})\\
& = \sum _{m=0}^p a_{i m}(\vq_i - \vk^*) b_{m j}(\vk_j-\vk^*) + \mathcal O(\delta ^{p+1}),
\end{align*}
where 
$$
a_{im}(\vq_i - \vk^*) = \frac{g(|\vq_i - \vk^*| )}{|\vq_i - \vk^*|^m}, \quad b_m (\vk_j - \vk^*)= \frac{g^{(m)}(1)}{m!}|\vk_j - \vk^*|^m.
$$
Let $\mU = (a_{im})$ and $\mV = (b_{jm})$. Then the desired result follows. 
\end{proof}

\section{Proof of Proposition~\ref{prop:rank-k}}
\begin{proof}[Proof of Proposition~\ref{prop:rank-k}]
Note that 
$$
\mL(\vx) = \phi_1(\vx)\phi_1(\vx)^\top + \phi_2(\vx)\phi_2(\vx)^\top + \cdots + \phi_r(\vx)\phi_r(\vx)^\top
$$
can be rewritten as
$$
\mL(\vx) = \mA(\vx)\mA(\vx)^\top,
$$
where $\mA(\vx)=[\phi_1(\vx),\phi_2(\vx),\cdots,\phi_r(\vx)]$. It is clear that ${\rm rank} \mA(\vx)=r$ since the columns of $\mA(\vx)$ are linearly independent. Therefore, ${\rm rank} \mL(\vx)=
{\rm rank}\mA(\vx)\mA(\vx)^\top = {\rm rank}\mA(\vx)=r.
$
\end{proof}

\section{Experimental Details}
Below we provide model and training details for our experiments in Section~\ref{sec:exp}. We use equation \eqref{eq:FMMformer-attention-blending} to blend the near-field and far-field attention and initialize the blending weights $w_1$ and $w_2$ to all-zero and all-one matrices, respectively. 

\subsection{Long Range Arena (LRA) Benchmark}
Our baselines consist of the linear transformer~\cite{katharopoulos2020transformers}, the vanilla standard softmax transformer~\cite{vaswani2017attention}, and the same vanilla standard softmax transformer but using a banded attention matrix of bandwidth 5. All models have 2 layers, 64 embedding dimension, 128 hidden dimension, 2 attention heads. Mean pooling is applied in all models. Also, we use the nonlinear activation ${\rm elu}(\vx) + 1$ for the linear transformer. For our models, we use the nonlinear activation ${\rm elu}(\vx) + 1$ for FMMformer 1-kernel and ${\rm elu}(\vx) + 1$ and ${\rm elu}(-\vx) + 1$ for FMMformer 2-kernel. As mentioned in Section~\ref{sec:exp}, in all FMMformer models, we use a banded attention matrix with bandwidth 5. 

Details on the Long Range Arena (LRA) benchmarks are given in the original paper~\cite{tay2021long}. Our implementation uses the public code by ~\cite{xiong2021nystromformer} as a starting point, and we follow their training procedures. The training setting and additional baseline model details are provided in the configuration file used in~\cite{xiong2021nystromformer} and can be found at \textcolor{blue}{\href{https://github.com/mlpen/Nystromformer/blob/main/LRA/code/lra_config.py}{https://github.com/mlpen/Nystromformer/blob/main/LRA/code}}. 

\subsection{Language Modeling on WikiText-103} \label{appendix:language-model-details}
Our language modeling implementation is based on the public code by~\cite{schlag2021linear}; we use their small configuration for all models in our experiment. In particular, we set the key, value, and query dimension to 128, and the training and evaluation context length to 256. We also set the number of heads to 8, the feed-forward layer dimension to 2048, and the number of layers to 16. For linear transformers and our FMMformer 1-kernel, we use the ${\rm elu}(\vx) + 1$  nonlinear activation, while for our FMMformer 2-kernel, we use ${\rm elu}(\vx) + 1$ and ${\rm elu}(-\vx) + 1$. Again, in all FMMformers, we use a banded attention matrix with bandwidth 5. 

All models are trained with the batch size of 96 for 120 epochs on two NVIDIA 3090Ti's with 24GB each. We apply 10\% dropout~\cite{hanson1990stochastic, srivastava2014dropout} and use the Adam optimizer~\cite{kingma2014adam} with an initial learning rate of 0.00025 and 2000 steps for learning rate warm-up. 

\section{Fast Weight Update Improves the Modeling of Far-Field Attention}

\begin{table}[t!]
    \begin{center}
    \begin{tabular}{ccc}
    \hline
        Method & Valid PPL & Test PPL \\
        \hline
        Softmax~\cite{vaswani2017attention} & 33.15 & 34.29 \\
        \hline\hline
        Linear~\cite{katharopoulos2020transformers}   & 37.27 & 38.40  \\
        Fast weight~\cite{schlag2021linear} & 35.75 & 36.63 \\
        Fast weight~\cite{schlag2021linear} + Linear~\cite{katharopoulos2020transformers} & 34.78 & 35.95 \\
        \hline\hline
        Band$_{20}$ & 38.18 & 39.19 \\
        \hline\hline
        FMMformer (1-kernel linear + Band$_{20}$) & 35.41  & 36.43   \\
        FMMformer (1-kernel fast weight + Band$_{20}$) & 34.54 &  35.47  \\
        \hline\hline
        FMMformer (2-kernel linear + Band$_{20}$) & 35.10  & 36.11  \\
        FMMformer (2-kernel fast weight + Band$_{20}$) & 34.16 &  34.71  \\
        \hline
    \end{tabular}
    \end{center}
    \caption{WikiText-103 language model perplexities of FMMformers compared to the baselines. The number of parameters (40\,M) is almost the same for all models, up to the small difference introduced by additional weights on the far-field attention in FMMformers. FMMformers outperform linear transformers~\cite{katharopoulos2020transformers}. The performance gap compared to softmax transformers is reduced when using a larger bandwidth in near-field attention and more kernels in far-field attention. Note that Band$_5$ and Band$_{20}$ are softmax transformers with a banded attention matrix of bandwidth 5 and 20, respectively.}
    \label{tab:wikitext103-fastweight}
\vspace*{-0.6em}
\end{table}

In Section~\ref{sec:exp}, we have explored using a linear transformer in our FMMformers to capture far-field attention. Though linear transformer is computational and memory efficient, its limited capacity might hinder the overall performance of FMMformers \cite{schlag2021linear}. As mentioned before, FMMformers are flexible in designing the low-rank matrix for modeling far-field attention. In this section, we investigate the potential of using a better linear model to capture far-field attention, thereby improving the advantage of FMMformers further. In particular, we consider using the linear transformer with a fast weight update proposed in~\cite{schlag2021linear} for modeling the far-field attention. 
We call the linear transformer with a fast weight update a fast weight transformer in this paper. The fast weight transformer is as efficient as the linear transformer but has better performance across popular benchmarks thanks to its better memory capacity compared to the original linear transformer~\cite{schlag2021linear}. We study the language modeling task on WikiText-103 and show that the advantage of the fast weight transformer over the linear transformer still maintains when they are integrated into our FMMformers, thus demonstrating the promise of FMMformers to achieve better performance when designed with powerful models that can capture better near-field and far-field attentions. Table~\ref{tab:wikitext103-fastweight} shows that FMMformers with fast weight attention outperforms FMMformers with linear attention. Furthermore, when using 2 kernels, the FMMformer with fast weight achieves performance close to the standard softmax transformer while being much more efficient since our FMMformer has linear computational and memory complexity. Note that our reproduced results for fast weight transformer alone in Table 3 is worse than the reported results in Table 2 in~\cite{schlag2021linear}. This is because in our fast weight transformer implementation, we use attention normalization so that the attention map of fast weight transformer is at the same scale as the attention map of the standard softmax and linear transformer. According to~\cite{schlag2021linear}, skipping attention normalization can help improve the results of fast weight transformer as shown in Table 3 in their paper. Exploring to use fast weight transformers without attention normalization in FMMformers is rather interesting, and we leave it for future work.

The training and model details in this experiment are the same as those in Section~\ref{appendix:language-model-details} above. When using two kernels in FMMformers, we use a fast weight transformer for one kernel and a linear transformer for another kernel. Both linear and fast weight transformers use ${\rm elu}(\vx) + 1$ for the nonlinear activation.

\section{Training and Validation Curves}

We demonstrate the training evolution of FMMformers compared to the baseline models including the linear and standard softmax transformer in Figure~\ref{fig:train-evolution-lm}. Here we consider the models trained for language modeling on WikiText-103 and show the training evolution in terms of train and validation perplexity during training. Our FMMformers have a faster convergence speed than both the linear transformer and the softmax transformer with a banded attention matrix.

\begin{figure*}[!t]
\centering
\includegraphics[width=\linewidth]{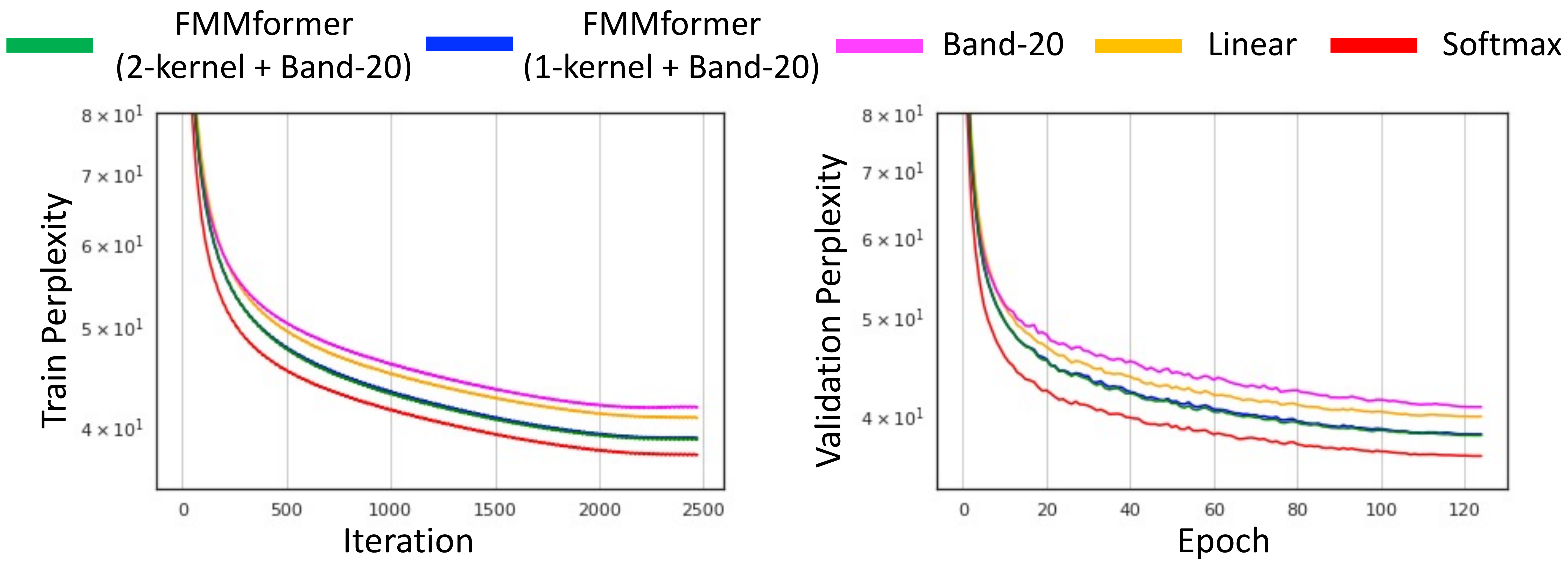}
\vskip -0.1cm
\caption{Train (left) and validation (right) perplexity during training for language modeling on WikiText-103. Our FMMformers converge faster than both the linear transformer and the softmax transformer with a banded attention matrix.}
\label{fig:train-evolution-lm}
\end{figure*}

\section{Visualization for FMMformer-based Language Models}
In this section, we visually study how attentions in FMMformers work. In particular, we consider the FMMformer with one low-rank kernel and one banded attention matrix of bandwidth 5 trained for the language modeling task on WikiText-103, i.e., the FMMformer (1-kernel + Band$_{5}$) in Table~\ref{tab:wikitext103}. We visualize the far-field and near-field attentions in one layer of the FMMformer in Figure~\ref{fig:lm-viz}. As expected, while the near-field attentions in FMMformers capture short-range dependencies, the far-field attentions capture long-range dependencies. Here the sequence length is 256, and the size of each matrix in Figure~\ref{fig:lm-viz} is $256 \times 256$.

\begin{figure*}[!t]
\centering
\includegraphics[width=\linewidth]{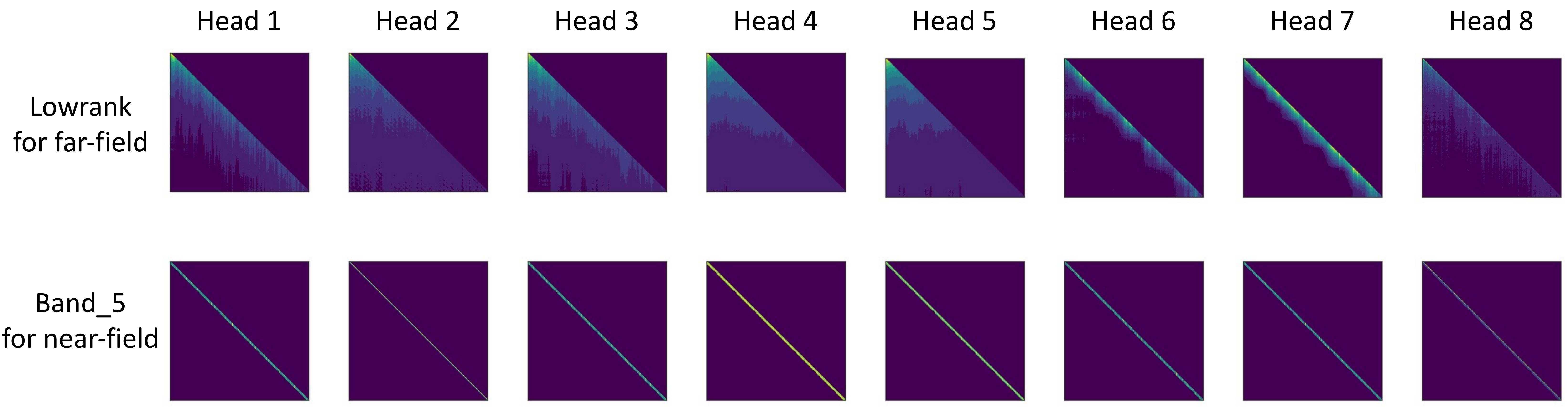}
\vskip -0.1cm
\caption{Attention matrices from each head at one layer of the FMMformer trained for language modeling on WikiText-103. The top row depicts the far-field attentions captured by the low-rank matrices in FMMformer, and the bottom row plots the near-field attentions captured by the banded matrices of bandwidth 5 in FMMformer. While the near-field attentions account for short-range dependencies in the data, the far-field attentions capture long-range dependencies. The sequence length here is 256, and the size of each matrix is $256 \times 256$}
\label{fig:lm-viz}
\end{figure*}

\end{document}